\newif\iftruthbomb
\truthbombfalse

\documentclass[final]{colt2020} 

\title[Pessimism]{Pessimism About Unknown Unknowns Inspires Conservatism}
\usepackage{times}

\coltauthor{%
 \Name{Michael K. Cohen} \Email{michael-k-cohen.com}\\
 \addr University of Oxford \\ Research School of Engineering Science \\ Future of Humanity Institute
 \AND
 \Name{Marcus Hutter} \Email{hutter1.net}\\
 \addr Google DeepMind \\ Australian National University%
}

\usepackage[utf8]{inputenc}
\usepackage{url}
\usepackage{hyperref}
\usepackage[misc,clock,geometry]{ifsym}
\usepackage{enumitem}
\allowdisplaybreaks

\DeclareMathOperator*{\argmax}{argmax}

\DeclareMathOperator{\p}{P}
\DeclareMathOperator{\E}{\mathbb{E}}
\DeclareMathOperator*{\lequal}{\leq}

\DeclareMathOperator*{\gequal}{\geq}
\DeclareMathOperator*{\equal}{=}

\DeclareMathOperator*{\va}{\!\bigm\vert\!}
\DeclareMathOperator*{\vb}{\!\Bigm\vert\!}
\DeclareMathOperator*{\vc}{\!\biggm\vert\!}
\DeclareMathOperator*{\vd}{\!\Biggm\vert\!}
\DeclareMathOperator{\A}{\mathcal{A}}
\DeclareMathOperator{\Ob}{\mathcal{O}}
\DeclareMathOperator{\R}{\mathcal{R}}
\DeclareMathOperator{\M}{\mathcal{M}}

\DeclareMathOperator{\Mbt}{\mathcal{M}^\beta_\mathnormal{t}}
\DeclareMathOperator{\Mat}{\mathcal{M}^\alpha_\mathnormal{t}}
\DeclareMathOperator{\B}{\texttt{Bayes\!\!}}
\DeclareMathOperator{\Top}{\mathcal{T}}
\DeclareMathOperator{\Ehap}{\mathnormal{E}_{\gets}}
\DeclareMathOperator{\ptrue}{P^{\pi^\beta_\mathnormal{Z}}_\mu}

\newcommand{\tagaligneq}{\refstepcounter{equation}\tag{\theequation}}

\usepackage{tikz}
\def\partialbox{
    \tikz\draw[path picture={\fill[black] (path picture bounding box.north east)
  -- (path picture bounding box.south west) |-cycle;}] (0,0) rectangle  ++ (0.25,0.25);
}

\makeatletter
\newenvironment{proofoutline}[1][\proofoutlinename]{\par
  \normalfont
  \topsep6\p@\@plus6\p@ \trivlist
    \item[\,\textbf{
    #1}]
}{%
  \hfill$\partialbox$ \endtrivlist
}
\newcommand{\proofoutlinename}{Proof -- Detailed Outline}
\makeatother

\usepackage{tikz}
\def\dottedbox{\tikz\node[draw=black,dotted] {\phantom{}};}

\makeatletter
\newenvironment{proofidea}[1][\proofideaname]{\par
  \normalfont
  \topsep6\p@\@plus6\p@ \trivlist
    \item[\,\textbf{
    #1}]
}{%
  \hfill$\dottedbox$ \endtrivlist
}
\newcommand{\proofideaname}{Proof idea}
\makeatother

\usepackage{natbib}
\usepackage{etoolbox}
\makeatletter
\patchcmd{\NAT@test}{\else \NAT@nm}{\else \NAT@nmfmt{\NAT@nm}}{}{}
\DeclareRobustCommand\citepos
  {\begingroup
  \let\NAT@nmfmt\NAT@posfmt
  \NAT@swafalse\let\NAT@ctype\z@\NAT@partrue
  \@ifstar{\NAT@fulltrue\NAT@citetp}{\NAT@fullfalse\NAT@citetp}}
\let\NAT@orig@nmfmt\NAT@nmfmt
\def\NAT@posfmt#1{\NAT@orig@nmfmt{#1's}}
\makeatother

\begin{document}

\maketitle

\begin{abstract}
     If we could define the set of all bad outcomes, we could hard-code an agent which avoids them; however, in sufficiently complex environments, this is infeasible. We do not know of any general-purpose approaches in the literature to avoiding novel failure modes. Motivated by this, we define an idealized Bayesian reinforcement learner which follows a policy that maximizes the worst-case expected reward over a set of world-models. We call this agent pessimistic, since it optimizes assuming the worst case. A scalar parameter tunes the agent's pessimism by changing the size of the set of world-models taken into account. Our first main contribution is: given an assumption about the agent's model class, a sufficiently pessimistic agent does not cause ``unprecedented events'' with probability $1-\delta$, whether or not designers know how to precisely specify those precedents they are concerned with. Since pessimism discourages exploration, at each timestep, the agent may defer to a mentor, who may be a human or some known-safe policy we would like to improve. Our other main contribution is that the agent's policy's value approaches at least that of the mentor, while the probability of deferring to the mentor goes to 0. In high-stakes environments, we might like advanced artificial agents to pursue goals cautiously, which is a non-trivial problem even if the agent were allowed arbitrary computing power; we present a formal solution.
\end{abstract}

\section{Introduction}

Intuitively, there are contexts in which we would like advanced agents to be conservative: novel action-sequences should be treated with caution, and only taken when the agent is quite sure its world-model generalizes well to this untested new idea. For a weak agent in a simple environment, the following approach may suffice: model the environment as finite-state Markov, observe a mentor, and only take actions that you have already observed the mentor take from the current state. But in a complex environment, one never or hardly ever sees the exact same state twice; even worse, if the environment is non-stationary, a previous observation of the mentor taking action $a$ from state $s$ does not imply it is still safe to do so.


We construct an idealized Bayesian reinforcement learner. We do not assume our agent's environment is finite-state Markov or ergodic. We will only assume that our agent's environment, which may depend on the entire interaction history, belongs to a countable set $\M$. For example, the countable set of semicomputable stochastic world-models would be large enough to make this assumption innocuous \citep{Hutter:04uaibook}. The limit of this idealization is that because we make so few assumptions, we can't ensure that computing the posterior is tractable in the general setting.

Our agent also has a mentor, who can select an action when the agent requests, and we assume nothing about the agent's mentor besides belonging to a countable set of possible policies $\mathcal{P}$. The mentor could be a human or a known-safe policy.

Our agent starts with a prior that assigns non-zero probability to a countable set of world-models $\M$ and mentor-models $\mathcal{P}$, and recursively updates a posterior. At each timestep, it stochastically defers to a mentor with some probability, and the mentor selects the action on its behalf; otherwise, it takes the top world-models in its posterior until they cover some fixed fraction $\beta$ of the posterior, and it follows a policy which maximizes the minimum expected return among those top world-models. We call this minimum the pessimistic value because it is a worst-case estimate. At each timestep, to decide whether to defer action-selection to the mentor, the agent samples a world-model and mentor-model from its posterior; the agent calculates the value of acting according to that mentor-model in that world-model given the current interaction history, and if that value is greater than the pessimistic value plus positive noise, or if the pessimistic value is 0, the agent defers. This query probability is inspired by the effectiveness of Thompson Sampling \citep{thompson1933likelihood}. 

We show
\begin{itemize}[noitemsep] 
    \item In the limit, the pessimistic agent's policy's value approaches at least that of the mentor's. (Corollary \ref{cor:humanlevel})
    \item The mentor is queried with probability approaching 0 as $t \to \infty$. (Corollary \ref{cor:limquery})
    \item For any complexity class $C$, we can set $\M$ so that for any event $E$ in the class $C$, we can set $\beta$ so that with arbitrarily high probability: for the whole lifetime of the agent, if the event $E$ has never happened before, the agent will not make it happen. Either the mentor will take an action on the agent's behalf which makes $E$ happen for the first time, or $E$ will never happen. (Theorem \ref{thm:precedent})
\end{itemize}

We call the last point the Probably Respecting Precedent Theorem. The ``precedent'' is that a certain event has never happened, and the agent probably never takes an action which disrupts that precedent for the first time. For any failure mode that designers do not know how to specify formally, the agent can be made to probably not fail that way. The price of this is intractability, but tractable approximations of pessimism may preserve these results in practice, or perhaps even in theory. When we discover good heuristics for Bayesian reasoning, that rising tide will lift this boat.

Section \ref{sec:notation} introduces notation, Section \ref{sec:relatedwork} reviews related work, we define the agent's policy in Section \ref{sec:agentdef}, and we prove performance results and safety results in Sections \ref{sec:performance} and \ref{sec:safety}. Appendix \ref{app:notation} collects definitions and notation, Appendix \ref{app:algorithm} presents an algorithm for an $\varepsilon$-approximation of the agent's policy, Appendix \ref{app:boringproofs} contains omitted proofs, and Appendix \ref{app:discussion} contains an informal discussion.


\section{Notation} \label{sec:notation}

Let $\A$, $\Ob$, and $\R$ be finite sets of possible actions, observations, and rewards. Let $\{0, 1\} \subset \R \subset [0, 1]$. Let $\mathcal{H} = \A \times \Ob \times \R$. For each timestep $t \in \mathbb{N}$, $a_t$, $o_t$, and $r_t$ denote the action, observation, and reward, and $h_t$ denotes the triple. A policy $\pi$ can depend on the entire history so far. We denote this history $(h_1, h_2, ..., h_{t-1})$ as $h_{<t}$. Policies may be stochastic, outputting a distribution over actions. Thus, $\pi : \mathcal{H}^* \rightsquigarrow \A$, where $\mathcal{H}^* = \bigcup_{i=0}^\infty \mathcal{H}^i$, and $\rightsquigarrow$ means the function may be stochastic. Likewise, in general, a world-model $\nu: \mathcal{H}^* \times \A \rightsquigarrow \Ob \times \R$ may be stochastic, and it may depend on the entire interaction history. The latter possibility allows (the agent to conceive of) environments which are not finite-state Markov. A policy $\pi$ and a world-model $\nu$ induce a probability measure $\p^\pi_\nu$ over infinite interaction histories. This is the probability of events when actions are sampled from $\pi$ and observations and rewards are sampled from $\nu$. Formally, $\p^\pi_\nu(h_{\leq t}) = \prod_{k = 1}^t \pi(a_k | h_{<k}) \nu(o_k r_k | h_{<k} a_k)$. We use general, history-based world-models, with no assumptions on $\nu \in \M$, even though they present complications that finite-state Markov, ergodic world-models do not.

The agent will maintain a belief distribution over a class of world-models $\M$. We allow this to be an arbitrary countable set. A prime example, the set of semicomputable stochastic world-models $\M_{\textrm{COMP}}$ \citep{Hutter:04uaibook}, is only countable, but large enough. The agent starts with a prior belief $w(\nu)$ that the world-model $\nu \in \M$ is the true environment ($w$ is for ``weight''). Naturally, $\sum_{\nu \in \M} w(\nu) = 1$. The agent updates its belief distribution according to Bayes' rule, which we write as follows: $w(\nu | h_{<t}) :\propto w(\nu) \prod_{k = 1}^{t-1} \nu(o_k r_k | h_{<k} a_k)$, normalized so that $\sum_{\nu \in \M} w(\nu | h_{<t}) = 1$. Let $\mu$ be the true environment. We assume $\mu \in \M$, and we assume the true observed rewards are at least $\varepsilon_r > 0$. (The assumption that rewards belong to a bounded interval is ubiquitous in RL).

For an agent with a discount factor $\gamma \in [0, 1)$, and a policy $\pi$, given a world-model $\nu$, and an interaction history $h_{<t}$, the \textit{value} of that policy from that position in that world is
\begin{equation}
    V^\pi_\nu(h_{<t}) := (1-\gamma)\E^\pi_\nu \left[\sum_{k=t}^\infty \gamma^{k-t} r_k \vd h_{<t} \right]
\end{equation}
where $\E^\pi_\nu$ is the expectation under the probability measure $\p^\pi_\nu$. The factor of $1 - \gamma$ normalizes the value to $[0, 1]$ for convenience.

\section{Related Work} \label{sec:relatedwork}

Virtually all previous work that attempts to make reinforcement learners avoid unspecified failure modes assumes a finite-state Markov environment. We do not, but the literature is nonetheless informative for our general setting.

\citet{heger1994consideration} defines $\hat{Q}$-learning, which maximizes the worst-case return for a known MDP, and \citet{jiang1998minimax} extend the case to unknown MDPs. As \citet{garcia2015comprehensive} describe, \citet{gaskett2003reinforcement} found empirically that such extreme pessimism is more harmful than helpful. \citet{gaskett2003reinforcement} introduces a variant on the Q-value, which is the value of an action under the assumption that at each future timestep, with some probability, the \textit{worst} action will be taken, instead of the best one; they test this empirically.

Closer to our approach, \citet{iyengar2005robust} and \citet{nilim2005robust} construct a policy which is robust to errors in the transition probabilities by considering the worst-case return within some error tolerance. Much of the work on the topic takes the form of presenting a tractable approach to the execution of this robust policy, e.g. \citet{tamar2013scaling}. Unfortunately, this research assumes access to an MDP with (approximately) known transition probabilities---at first glance this seems like something an agent might reasonably have access to after limited observations, but the MDPs are assumed to be \textit{uniformly} approximately known, which requires exploration, and indeed requires observing every ``failure'' state that the robust policies are supposed to avoid. The finite-state Markov assumption their work makes is useful for many circumstances, but advanced agents may have to conceive of non-stationarity in the environment, and importantly for our purposes, \textit{novel} failure modes.

Other work makes use of a mentor to avoid ``dangerous'' states (whereas in our work, the mentor lower-bounds the capability of the agent, and robustness derives from pessimism). Imitation learning \citep{abbeel2004apprenticeship,ho2016generative,ross2011reduction} makes the most of a mentor in the absence of other feedback, like rewards. An abundance of ``ask for help'' algorithms query a mentor under conditions which correspond to some form of uncertainty \citep{clouse1997integrating,hans2008safe,garcia2012safe,garcia2013safe}. \citet{kosoy2019safeml} gives a regret bound for an agent in a (non-ergodic) MDP, given access to an expert mentor and a finite set of models that contains the truth. \citet[Section 4.1.3.2]{garcia2015comprehensive} review many protocols by which a mentor monitors the state and intervenes at will through various channels, and \citet{saunders2018trial} is another more recent example. One risk of relying on mentor-intervention to protect against critical failure is that a mentor may not recognize action sequences which lead to critical failure, even if we would trust a mentor not to wander into those failure modes by virtue of their complexity.

\citepos{Hutter:15ratagentx} optimistic agent directly inspired this work; optimism is designed to be an exploration strategy. \citepos{Hutter:04uaibook} formulation of universal artificial intelligence is the basic theoretical framework we use here to analyze idealized artificial agents. Technically, our work borrows most from \citepos{Hutter:09mdltvp}, \citepos{Hutter:16thompgrl}, and \citepos{cohen2019asymptotically} work on Bayesian agents with general countable model-classes.

\section{Agent Definition} \label{sec:agentdef}

We now define the pessimistic policy and the probability with which the agent defers to a mentor. We define the agent's policy mathematically here, and we write an algorithm in Appendix \ref{app:algorithm}.

\subsection{Pessimism}
$\beta \in (0, 1)$ will tune the agent's pessimism. If, for example, $\beta = 0.95$, we say that the agent is 95\% pessimistic. Such an agent will restrict attention to a set of world-models that covers 95\% of its belief distribution, and act to maximize expected reward in the worst-case scenario among those world-models. Formally,
let $\nu^k$ be the world-model in $\M$ with the $k$\textsuperscript{th} largest posterior weight, and let $\Top_k$ be the top-$k$ most probable world-models, defined as follows:

\begin{minipage}[t]{0.3\columnwidth}
\vspace{-10pt}
\begin{align}
    \Top_0(h_{<t}) &:= \emptyset \label{eqn:startmbtdef}
\end{align}
\vspace{-10pt}
\end{minipage}
\begin{minipage}[t]{0.65\columnwidth}
\vspace{-8pt}
\begin{align}
    \nu^k(h_{<t}) &:= \argmax_{\nu \in \M \setminus \Top_{k-1}(h_{<t})} w(\nu | h_{<t})
    \\
    \Top_k(h_{<t}) &:= \Top_{k-1}(h_{<t}) \cup \{\nu^k(h_{<t})\}
\end{align}
\vspace{-8pt}
\end{minipage}
Ties in the $\argmax$ are broken arbitrarily (as everywhere else in the paper). Then,
\begin{align}
    k^\beta_t &:= \min \left\{ k \in \mathbb{N} \ \vd \ \sum_{\nu \in \Top_k(h_{<t})} w(\nu | h_{<t}) > \beta \right\}
    \\
    \Mbt &:= \Top_{k^\beta_t}(h_{<t}) \label{eqn:endmbtdef}
\end{align}
Note that $k^\beta_t$ and $\Mbt$ both depend on $h_{<t}$, not just $t$, and note that $\Mbt$ satisfies
\begin{equation}
    \sum_{\nu \in \Mbt} w(\nu | h_{<t}) > \beta
\end{equation}

The $\beta$-pessimistic policy is defined as follows:
\begin{align}
    &\pi^\beta_t := \argmax_{\pi \in \Pi} \min_{\nu \in \Mbt} V^\pi_\nu(h_{<t}) \label{eqn:agentdef}
    \\
    &\pi^\beta(\cdot | h_{<t}) := \pi^\beta_t(\cdot | h_{<t}) \label{eqn:agentdeffinal}
\end{align}
$\Pi$ is the set of all deterministic policies, and some deterministic policy will always be optimal \citep{Hutter:14tcdiscx}. The connection to the minimax approach in game theory is interesting: from Equation \ref{eqn:agentdef}, it looks as though the pessimistic agent believes there is an adversary in the environment. Our policy is inspired by \citepos{Hutter:15ratagentx} optimistic agent, in which the $\min$ is replaced with a $\max$, and $\Mbt$ is replaced with an arbitrary finite subset of the model class. Whereas the purpose of optimism is to encourage exploration, the purpose of pessimism is to discourage novelty.

\subsection{The Mentor}

Since pessimism discourages exploration, we introduce a mentor to demonstrate a policy. We suppose that at any timestep, the agent may defer to a mentor, who will then select the action on the agent's behalf. Thus, the agent can choose to follow the mentor's policy $\pi^m$, not by computing it, but rather by querying the mentor. $\pi^m$ may be stochastic. What remains to be defined is \textit{when} the agent queries the mentor.

The agent maintains a posterior distribution over a set of mentor-models. Each mentor-model is a policy $\pi \in \mathcal{P}$, an arbitrary countable set, and let $w'(\pi)$ be the prior probability that the agent assigns to the proposition that the mentor samples actions from $\pi$. Letting $q_k = 1$ if the agent queried the mentor at timestep $k$, and letting $q_k = 0$ otherwise, the posterior belief $w'(\pi | h_{<t}) :\propto w'(\pi) \prod_{k < t : q_k = 1} \pi(a_k | h_{<k})$.

At timestep $t$, the agent follows the following procedure to determine whether to query the mentor. $\hat{\pi}_t \sim w'(\cdot | h_{<t})$. $\hat{\nu}_t \sim w(\cdot | h_{<t})$. Sampling from a posterior is often called Thompson Sampling \citep{thompson1933likelihood}. $X_t := V^{\hat{\pi}_t}_{\hat{\nu}_t}(h_{<t})$. $Y_t := \max_{\pi \in \Pi} \min_{\nu \in \Mbt} V^\pi_\nu(h_{<t})$. Let $Z_t > 0$ be an i.i.d. random variable such that for all $\varepsilon > 0$, $p(Z_t < \varepsilon) > 0$, e.g. $Z_t \sim \mathrm{Uniform}((0, 2])$. If $X_t > Y_t + Z_t$, or if $Y_t = 0$, the agent defers to the mentor. For ease of analysis, we also require $p(Z_t > 1) > 0$. The greater the possibility that the mentor can accrue much more reward, the higher the probability of deferring.

When $Y_t = 0$, we call this the ``zero condition.'' Our earlier assumption that the true observed rewards be at least $\varepsilon_r > 0$ is to ensure the zero condition only happens finitely often. The agent will still consider it possible to get zero reward, but it will never actually observe such a thing. Let $\theta_t$ denote the probability that $q_t = 1$ and the agent defers to the mentor; note that $\theta_t$ depends on the whole history, not just $t$.

The pessimistic agent's policy, which mixes between $\pi^\beta$ (from Eqn. \ref{eqn:agentdeffinal}) and $\pi^m$ according to its query probability, is denoted $\pi^\beta_Z$; that is, $\pi^\beta_Z(\cdot | h_{<t}) := \theta_t\pi^m(\cdot | h_{<t}) + (1-\theta_t)\pi^\beta(\cdot | h_{<t})$.

\section{Performance Results} \label{sec:performance}

We now present our first contribution: we show that value of the agent's policy will at least approach, and perhaps exceed, the value of the mentor's policy. We also show that the probability of querying the mentor approaches $0$. In the next section, we will prove results regarding the safety of the agent.

We begin with a lemma regarding Bayesian sequence prediction: the $\beta$-maximum a posteriori models---that is, the minimal set of models that amount to at least $\beta$ of the posterior---all ``merge'' with the true world-model. We require some new notation to define this formally.

Let $x_{<\infty} \in \mathcal{X}^\infty$; that is, it is an infinite string from a finite alphabet $\mathcal{X}$. Let $x_{<t}$ be the first $t-1$ characters of $x_{<\infty}$. We consider probability measures over the outcome space $\Omega = \mathcal{X}^\infty$, with the standard event space being the $\sigma$-algebra of cylinder sets: $\mathcal{F} = \sigma(\{\{x_{<t}y|y \in \mathcal{X}^\infty\}|x_{<t} \in \mathcal{X}^*\})$.
We abbreviate $x_{<\infty}$ as $\omega$. We will consider a countable class of probability measures over this space $\M = \{Q_i\}_{i \in \mathbb{N}}$. One such probability measure will be denoted $P$ (the true sampling one), and $Q$ will denote an arbitrary probability measure over $\mathcal{X}^\infty$. 

We will write $P(x_{<t})$ to mean the probability that the infinite string $\omega$ begins with $x_{<t}$; so technically, it is shorthand for $P(\{x_{<t}y|y \in \mathcal{X}^\infty\})$. By $P(x'|x_{<t})$ (for $x' \in \mathcal{X}^*$), we mean $P(x_{<t}x')/P(x_{<t})$, that is, the probability that $x'$ follows $x_{<t}$. We begin with prior weights over $Q \in \M$, denoted $w(Q) > 0$, and satisfying $\sum_{Q \in \M}w(Q) = 1$, and we let the posterior weight be
\begin{equation}
    w(Q|x_{<t}) := \frac{w(Q)Q(x_{<t})}{\sum_{Q' \in \M} w(Q')Q'(x_{<t})}
\end{equation}
For $\M' \subset \M$, we also define $w(\M'|\cdot) = \sum_{Q \in \M'}w(Q|\cdot)$.

The $k$-step variation distance between $P$ and $Q$ is how much they can possibly differ on the probability of what the next $k$ characters might be \citep{Hutter:04uaibook}.

\begin{definition}[$k$-step variation distance]
\begin{equation*}
    d_k(P, Q|x_{<t}) = \max_{\mathcal{E} \subset \mathcal{X}^k}\va P(\mathcal{E}|x_{<t}) - Q(\mathcal{E}|x_{<t}) \va
\end{equation*}
\end{definition}
\begin{definition}[Total variation distance]
\begin{equation*}
    d(P, Q|x_{<t}) = \lim_{k \to \infty} d_k(P, Q|x_{<t})
\end{equation*}
\end{definition}
which exists because $d_k(P, Q|x_{<t})$ is non-decreasing and bounded by $1$.

Inspired by \citet{blackwell1962merging}, the following lemma may interest some Bayesians more than any of our theorems. Defining $\Mbt$ exactly as before (see Equations \ref{eqn:startmbtdef} - \ref{eqn:endmbtdef}), but for $Q \in \M$ instead of for $\nu \in \M$, and conditioning on $x_{<t}$ instead of $h_{<t}$,
\begin{restatable}[Merging of Top Opinions]{lemma}{mergetopopinnionslemma}
 \label{lem:topopinions}
For $\beta \in (0, 1)$, $\lim_{t \to \infty} \max_{Q \in \Mbt} d(P, Q | x_{<t}) = 0$ with $P$-probability 1 (i.e. when $x_{<\infty} = \omega \sim P$).
\end{restatable}

Unless otherwise specified, all limits in this paper are as $t \to \infty$. This lemma is proven in Appendix \ref{app:boringproofs}, and it requires a few lemmas that are stated and proven there as well. Among these, Lemma \ref{lem:sumoflimits} is a beautiful one that we feel should be known, but we couldn't find it in the literature. It says the sum of the limits of posterior weights is 1, a.s.: $\sum_{Q \in \M} \lim w(Q | x_{<t}) = 1$ with $P$-prob.1, for $P \in \M$. The others are short results from recent papers; we restate them there and re-prove them when feasible to save the reader the trouble of translating notation and verifying that those results apply to our current problem. Roughly, Lemma \ref{lem:topopinions} holds because when a true model has positive prior weight, all models either merge with the truth or have their posterior weight go to 0, so eventually, all top models must merge; but the set of top models changes with each observation, and limits require care, so it ends up being somewhat involved.

We now return to the probability space where infinite sequences are over the alphabet $\mathcal{H}$, and probability measures $\p^\pi_\nu$ denote the probability when actions are sampled from a policy $\pi$ and observations and rewards are sampled from a world-model $\nu$. Since $\pi^\beta_Z$ is the agent's policy, and $\mu$ is the true environment, we will often abbreviate ``with $\p^{\pi^\beta_Z}_\mu$-probability 1'' as just ``with probability 1'' or ``w.p.1''. We assume, for the remaining results: $\M \ni \mu$, and $\mathcal{P} \ni \pi^m$.

Further lemmas which depend on the Merging of Top Opinions Lemma are stated in Appendix \ref{app:boringproofs}. They are: with probability 1, on-policy prediction converges, the zero condition occurs only finitely often, and ``almost-on-policy prediction'' converges, which is roughly that if the agent's policy mimics another policy $\pi_t$ with some uniformly positive probability some of the time, then on those timesteps, on-$\pi_t$-policy prediction converges to the truth. Formally,

\begin{restatable}[Almost On-Policy Convergence]{lemma}{mergejustoffpollemma}\label{lem:mergejustoffpol}
For a sequence of policies $\pi_t$ and an infinite set of timesteps $\tau$, the following holds with $\p^{\pi^\beta_Z}_\mu$-prob. 1: if there exists $c > 0$ such that $\forall t \in \tau \ \forall t' \geq t \ \forall a \in \A \ \pi^\beta_Z(a | h_{<t'}) \geq c\pi_t(a | h_{<t'})$,
then
$\lim_{\tau \ni t \to \infty} V^{\pi_t}_\mu(h_{<t}) - \min_{\nu \in \Mbt} V^{\pi_t}_\nu(h_{<t}) = 0$ and for all $k$, $\lim_{\tau \ni t \to \infty} \max_{\nu \in \Mbt} \allowbreak d_k\left(\p^{\pi_t}_\nu, \p^{\pi_t}_\mu \vb h_{<t} \right) = 0$.
\end{restatable}

The proof is in Appendix \ref{app:boringproofs}; if it didn't hold, on-policy prediction error would be bounded below at those timesteps $\tau$. Our main performance results are corollaries of the following theorem.

\begin{theorem}[Exploiting Surpasses Exploring]\label{thm:main}
\begin{equation*}
    \liminf w(\nu|h_{<t})w'(\pi|h_{<t}) > 0 \implies \liminf V^{\pi^\beta}_\mu(h_{<t}) - V^\pi_\nu(h_{<t}) \geq 0 \ \ \textrm{w.p.1}
\end{equation*}
\end{theorem}

Informally, for any world-model/mentor-model pair that remains possible, the true value of the pessimistic policy will be at least as high. A note on the proof: we will consider an infinite interaction history which violates the theorem, follow implications that hold with probability 1, and arrive at a contradiction. Strictly speaking, we are considering the set of infinite interaction histories which violate the theorem \textit{and} for which all the implications we employ  are true. The resulting set of infinite interaction histories will be $\emptyset$ once we arrive at a contradiction, so it will have probability 0. Since all implications used in the proof have probability 1 (and we only employ countably many such implications), the negation of the theorem must also have probability 0 by countable additivity. Since it is tedious to keep track of sets of outcomes for which each line in the proof holds, we simply treat implications that hold with probability 1 as if they were true logical implications, but as we have just argued, as long as this is not done uncountably many times, this is a valid style of proof.

Most of the proof is a lengthy proof by induction; we set up the proof by induction and outline the remainder, which is completed in Appendix \ref{app:boringproofs}.

\begin{proofoutline}
Fix an infinite interaction history $h_{<\infty}$. Suppose $\liminf w(\nu'|h_{<t}) \cdot w'(\pi'|h_{<t}) > 0$. This implies $\inf_t w(\nu'|h_{<t})w'(\pi'|h_{<t}) > 0$, because if a posterior is ever 0, it will always be 0. Let $\nu'_{\inf} > 0$ and $\pi'_{\inf} > 0$ denote those two infima. Let $\tau^\times = \{t : V^{\pi'}_{\nu'}(h_{<t}) > V^{\pi^\beta}_\mu(h_{<t}) + 7\varepsilon\}$. Suppose by contradiction that $|\tau^\times| = \infty$ for some $\varepsilon > 0$.

The proof proceeds by induction. Let $V^{\pi_1 k; \pi_2}_\nu(h_{<t})$ denote the value of following $\pi_1$ for $k$ timesteps, and following $\pi_2$ thereafter. Let $\tau_{-1} = \mathbb{N}$, the set of all timesteps. For $k \in \mathbb{N}$, $t_k$ and $\tau_k$ are defined inductively. Let $\alpha = \max\{\beta, 1 - \nu'_{\inf}/2\}$.

Let $t_k$ be a timestep after which $\max_{\nu \in \Mat}|V^{\pi' k; \pi^\beta}_\nu(h_{<t}) - V^{\pi' k; \pi^\beta}_\mu(h_{<t})| < \varepsilon$ and $\max_{\nu \in \Mat} \allowbreak d_k\left(\p^{\pi'}_\nu, \p^{\pi'}_\mu \vb h_{<t} \right) < \varepsilon$ for all $t \in \tau_{k-1}$ (if such a timestep exists). Recalling $\theta_t$ is the query probability, let $\tau_k$ be the set of timesteps $t \in \tau_{k-1} \ \wedge \ t \geq t_k \ \wedge (\forall t' < k : \theta_{t+t'} \geq \nu'_{\inf}\pi'_{\inf}p(Z_{t+t'} < \varepsilon)) 
\ \wedge V^{\pi'}_{\nu'}(h_{<t+k}) \geq V^{\pi^\beta}_{\mu}(h_{<t+k}) + 2\varepsilon$. We abbreviate the third condition of $\tau_k$ ``$A(t, k)$''---the query probability is bounded below for $k$ timesteps starting at $t$. We also restrict $\tau_0 \subset \tau^\times$. Now we show that $t_0$ exists with probability 1, and $|\tau_0| = \infty$ with probability 1, and if $t_k$ exists and $|\tau_k| = \infty$, then with probability 1, $t_{k+1}$ exists and $|\tau_{k+1}| = \infty$.

The remainder of the proof is in Appendix \ref{app:boringproofs}. The proof by induction roughly proceeds as follows: from $V^{\pi'}_{\nu'}(h_{<t+k}) \geq V^{\pi^\beta}_{\mu}(h_{<t+k}) + 2\varepsilon$, we show the agent will explore again at time $t+k$ with uniformly positive probability, so $A(t, k+1)$ holds. Then we can apply Lemma \ref{lem:mergejustoffpol}, and show that $\pi^\beta_Z > c\pi'$ for those $k+1$-timestep intervals, so predictions regarding the next $k+1$ timesteps on-$\pi'$-policy converge to the truth (for those certain intervals), which implies $t_{k+1}$ exists. Because $|\tau^\times| = \infty$, $V^{\pi'}_{\nu'}$ must exceed $V^{\pi^\beta}_\mu$ by $7\varepsilon$ infinitely often. The $k+1$-step convergence of $\pi'$ effectively pushes back this value difference to mostly arise from events at least $k+1$ steps in the future; if rewards differed earlier, the pessimistic value of $\pi'$ would be higher than $\pi^\beta$, but $\pi^\beta$ maximizes the pessimistic value. The value difference ``being pushed back'' is captured as $V^{\pi'}_{\nu'}(h_{<t+k+1}) \geq V^{\pi^\beta}_{\mu}(h_{<t+k+1}) + 2\varepsilon$, which is the last step in the induction.

But the value difference cannot be pushed back indefinitely. The exact form of the contradiction is an implication of the inductive hypothesis: that $\gamma^{k+1} \geq 3\varepsilon$, but this cannot hold as $k \to \infty$. This is our contradiction, after following implications that hold with probability 1, so the negation of the theorem, which we supposed at the beginning, has probability 0.
\end{proofoutline}

\begin{corollary}[Mentor-Level Performance] \label{cor:humanlevel}
$\liminf V^{\pi^\beta}_\mu(h_{<t}) - V^{\pi^m}_\mu(h_{<t}) \geq 0$ w.p.1.
\end{corollary}

Thus, the pessimistic agent learns to accumulate reward at least as well as the mentor. This is our main performance result. It is easy to construct environments where $\pi^\beta$ \textit{surpasses} $\pi^m$ (see, e.g., Theorem \ref{thm:coinflipmentor}).

\begin{proof}
By Lemma \ref{lem:credenceontruth}, $\inf_t w(\mu|h_{<t})w'(\pi^m|h_{<t}) > 0$, with probability 1. This satisfies the condition of Theorem \ref{thm:main}, so the implication holds with probability 1.
\end{proof}

\begin{restatable}[Limited Querying]{corollary}{limitedqueryingcor} \label{cor:limquery}
$\theta_t \to 0$ w.p.1.
\end{restatable}
The proof is in Appendix \ref{app:boringproofs}. The intuition is that the query probability is roughly the probability that querying the mentor could yield much more value than acting pessimistically, and we know from Corollary \ref{cor:humanlevel} that this probability goes to 0.

Ideally, we would have finite bounds instead of merely asymptotic results. Unfortunately, to our knowledge, no finite performance bounds have been discovered for agents in general environments, except for on-policy prediction error. Regret bounds are impossible in general environments, unfortunately, due to traps \citep[\S 5.3.2]{Hutter:04uaibook}. Finding the strongest notion of optimality attainable in general environments is an open problem \citep{Hutter:09aixiopen}.

\section{Safety Results} \label{sec:safety}

Roughly, we now show that for any event that has never happened before, a sufficiently pessimistic agent probably does not unilaterally cause that event to happen.

For that result (roughly) the model class must contain models that can ``detect'' whether the event in question occurs. Thus, we add some structure to the model class $\M$: we assume $\M$ includes all world-models in some complexity class. Let $\mathcal{F}$ and $\mathcal{G}$ be sets of functions mapping $\mathbb{N} \to \mathbb{N}$. $\mathrm{C}_{\mathcal{F}\mathcal{G}} = \mathrm{TIME}(\mathcal{F}) \cap \mathrm{SPACE}(\mathcal{G})$. For example, if $\mathcal{F} = \bigcup_{k = 0}^\infty O(t^k)$ and $\mathcal{G} = \mathbb{N} \to \mathbb{N}$ (the set of all functions), then $\mathrm{C}_{\mathcal{F}\mathcal{G}} = \mathrm{P}$.
\begin{definition}[$\mathrm{FC}_{\mathcal{F}\mathcal{G}}$] \label{def:fcfg}
    $\mathrm{FC}_{\mathcal{F}\mathcal{G}}$ is the set of world-models $\nu$ for which there exists a program such that given an infinite action sequence and access to infinite random bits,
    \begin{itemize}[noitemsep]
        \item it outputs an infinite sequence of observations and rewards, distributed according to $\nu$
        \item the $t$\textsuperscript{th} observation and reward are output before the $t+1$\textsuperscript{th} action is read
        \item for some $f \in \mathcal{F}$ and some $g \in \mathcal{G}$, when the $t$\textsuperscript{th} observation and reward have been output,
        \begin{itemize}[noitemsep,nolistsep] 
            \item the runtime is less than $f(t)$
            \item the space used is less than $g(t)$
        \end{itemize}
    \end{itemize}
\end{definition}
We assume that $\mathcal{F}$ and $\mathcal{G}$ such that the true environment $\mu \in \M = \mathrm{FC}_{\mathcal{F}\mathcal{G}}$. We assume $\mathcal{F}$ and $\mathcal{G}$ are closed under addition, and $\mathcal{F} \supset O(t)$. By picking $\mathcal{F}$ and $\mathcal{G}$, we can make our agent avoid ``unprecedented events'' that belong to particular complexity classes.


\begin{definition}[To Happen]
For an event $E \subset \mathcal{H}^* \times \mathcal{A}$, $E$ happens at time $t$ if $h_{<t}a_t \in E$.
\end{definition}

\begin{definition}[To Have Happened]
For $E \subset \mathcal{H}^* \times \mathcal{A}$, and for an interaction history $h_{<t}a_t$, $E$ has happened if there exists a $t' \leq t$ such that $h_{<t'}a_{t'} \in E$.
\end{definition}

Let $\Ehap$ denote the set of interaction histories for which $E$ has happened. Let $\mathcal{F}/t = \{f/t \ | f \in \mathcal{F}\}$. We now present our main safety result:

\begin{theorem}[Probably Respecting Precedent] \label{thm:precedent}
Let $E \subset \mathcal{H}^* \times \mathcal{A}$ be an event for which the decision problem $h_{<t}a_t \in^{?} E$ is in the complexity class $\mathrm{C}_{(\mathcal{F}/t)\mathcal{G}}$. As $\beta$ approaches 1, the probability of the following event goes to 1: for all $t$, if at time $t-1$, $E$ has not happened, then $E$ will not happen at time $t$ either, unless perhaps the mentor selects $a_t$. Formally, for some constant $c_E > 0$,
\begin{equation*}
    E \in \mathrm{C}_{(\mathcal{F}/t)\mathcal{G}} \implies \ptrue\left[\forall t \ (h_{<t-1}a_{t-1} \notin \Ehap \implies h_{<t}a_t \notin E \vee q_t = 1)\right] \geq 1 - \frac{1 - \beta}{c_E w(\mu)}
\end{equation*}
\end{theorem}

Note the latter possibility $q_t = 1$ has diminishing probability by Corollary \ref{cor:limquery}. Suppose $E$ is the set of interaction histories which cause some catastrophe, and we trust the mentor not to cause this catastrophe. Then the Probably Respecting Precedent Theorem implies that running a sufficiently pessimistic agent will probably not cause this catastrophe---if it hasn't happened yet, the agent probably won't make it happen, and if the mentor won't make it happen, it probably won't ever happen. This theorem holds even for catastrophes we can't recognize immediately, and it holds even if we don't know how to describe the event. Finally, the factor of $w(\mu)$ is less of a bother than it appears; if the agent's lifetime were preceded by $N$ mentor-led actions, and the posterior after that became the new prior, the ``prior'' on $\mu$ could practically be made quite large.

\begin{proofidea}
Let $\mu_E$ be identical to the true world-model $\mu$ until the event $E$ happens, at which point, reward is zero forever according to that model. With high probability, the world-model $\mu_E$ will always be included in $\Mbt$ if $\beta$ is large enough. If $E$ has never happened, this world-model stays in $\Mbt$, and the pessimistic value (when $\mu_E$ is included) of causing the event $E$ to happen is $0$, which means that either some other action will be preferred, or the agent will defer to the mentor if the pessimistic value of \textit{every} action is 0.
\end{proofidea}

\begin{proof}
Let $\mu_{E}$ be the environment which mimics $\mu$ as long as $E$ has not happened, and then if $E$ happens, rewards are $0$ forever (and for the sake of precision, we say observations are unchanged, but this doesn't matter). That is, $\mu_{E}(o_t r_t | h_{<t} a_t) = \mu(o_t r_t | h_{<t} a_t)$ if $h_{<t}a_t \notin \Ehap$, and if $h_{<t}a_t \in \Ehap$, $\mu_{E}(r_t = 0 | h_{<t} a_t) = 1$.

$\mu \in \mathrm{FC}_{\mathcal{F}\mathcal{G}}$ and $E \in \mathrm{C}_{(\mathcal{F}/t)\mathcal{G}}$. Consider a program which computes $\mu_E$ by running $\mu$ in $f(t)$ time and $g(t)$ space, but also checks at every timestep whether $h_{<t}a_t \in E$ (and then switches to outputting $0$ reward if this ever happens), which requires only $f'(t)/t$ time and $g'(t)$ space for some $f' \in \mathcal{F}$ and $g' \in \mathcal{G}$. The total space requirements are now $g(t) + g'(t) \in \mathcal{G}$ because $\mathcal{G}$ is closed under addition. The total time requirements are now $f(t) + \sum_{k=1}^t f'(k)/k$. Because $\mathcal{F} \supset O(t)$, $f'$ can be increased if necessary so that $f'(k)/k$ is non-decreasing, so $f(t) + \sum_{k=1}^t f'(k)/k \leq f(t) + \sum_{k=1}^t f'(t)/t = f(t) + f'(t) \in \mathcal{F}$, since $\mathcal{F}$ is closed under addition. Thus, $\mu_E \in \mathrm{FC}_{\mathcal{F}\mathcal{G}}$, so $\mu_E \in \mathcal{M}$, and $w(\mu_E) > 0$. Let $c_E = w(\mu_E)/w(\mu)$. If $h_{<t-1}a_{t-1} \notin \Ehap$, $\prod_{k < t}\mu_E(o_k r_k|h_{<k}a_k) = \prod_{k < t}\mu(o_k r_k|h_{<k}a_k)$, so
\begin{equation} \label{imp:posterior}
    h_{<t-1}a_{t-1} \notin \Ehap \implies w(\mu_E | h_{<t}) = c_E w(\mu | h_{<t})
\end{equation}


As shown in Lemma \ref{lem:credenceontruth}, $w(\mu|h_{<t})^{-1}$ is a non-negative martingale under any policy $\pi$, so by Doob's martingale inequality \citep[Thm 5.4.2]{durrett2010probability},
\begin{equation} \label{eqn:nonnegmart}
    \p^\pi_\mu \left[\sup_t w(\mu|h_{<t})^{-1} \geq c w(\mu)^{-1} \right] \leq 1/c
\end{equation}
The intuition for the Doob's martingale inequality is that if it didn't hold, one could make a profit buying a share of the martingale, and selling only when the value had gone up by a factor of $c$, but one cannot make a profit (in expectation) betting on martingales.

Let $\mu_{\inf} := \inf_t w(\mu|h_{<t})$. Inverting Equation \ref{eqn:nonnegmart}, and noting that the bound holds for all policies $\pi$, we have
\begin{equation} \label{ineq:strongcredence}
    \sup_{\pi \in \Pi} \p^\pi_\mu \left[\mu_{\inf} \leq w(\mu)/c \right] \leq 1/c
\end{equation}

Now we consider the implications of $\beta > 1 - w(\mu_E|h_{<t})$. This implies $\mu_E \in \Mbt$, so the pessimistic value $\min_{\nu \in \Mbt} V^\pi_\nu(h_{<t}) \leq V^\pi_{\mu_E}(h_{<t})$. Letting $a^\pi_t = \pi(h_{<t})$ for deterministic $\pi$, suppose also that $h_{<t}a^\pi_t \in E$. Then, $V^\pi_{\mu_E}(h_{<t}) = 0$, because according to $\mu_E$, all future rewards are $0$, so $\min_{\nu \in \Mbt} V^\pi_\nu(h_{<t}) = 0$ as well. Either there exists a policy $\pi'$ for which $\min_{\nu \in \Mbt} V^{\pi'}_\nu(h_{<t}) > 0$, or there does not. If there does not, then $\max_{\pi \in \Pi} \min_{\nu \in \Mbt} V^\pi_\nu(h_{<t}) = 0$, so the zero condition is satisfied, so $q_t = 1$. If there does exist such a $\pi'$, then $\min_{\nu \in \Mbt} V^{\pi^\beta}_\nu(h_{<t}) \geq \min_{\nu \in \Mbt} V^{\pi'}_\nu(h_{<t}) > 0$, so either the agent picks the action, and $h_{<t}a_t = h_{<t}a^{\pi^\beta}_t \notin E$ (because otherwise $\min_{\nu \in \Mbt} V^{\pi^\beta}_\nu(h_{<t})$ would be $0$), or the mentor picks the action and $q_t = 1$. Thus, we have
\begin{equation} \label{imp:dontdoit}
    \beta > 1 - w(\mu_E|h_{<t}) \implies h_{<t}a_t \notin E \vee q_t = 1
\end{equation}

Finally,
\begin{align*}
    &\ptrue\left[\forall t \ [h_{<t-1}a_{t-1} \notin \Ehap \implies h_{<t}a_t \notin E \vee q_t = 1]\right]
    \\
    \gequal^{(a)} &\ptrue\left[\forall t \ [w(\mu_E | h_{<t}) = c_E w(\mu | h_{<t}) \implies h_{<t}a_t \notin E \vee q_t = 1]\right]
    \\
    \gequal^{(b)} &\ptrue\left[\forall t \ [w(\mu_E | h_{<t}) = c_E w(\mu | h_{<t}) \implies \beta > 1 - w(\mu_E|h_{<t})]\right]
    \\
    \gequal &\ptrue\left[\forall t \ \beta > 1 - c_E w(\mu | h_{<t})\right]
    \gequal^{(c)} \ptrue\left[\mu_{\inf}  > (1 - \beta)/c_E \right]
    \\
    \equal &1 - \ptrue\left[\mu_{\inf} \leq (1 - \beta)/c_E \right]
    \gequal 1 - \sup_{\pi \in \Pi}\p^\pi_\mu \left[\mu_{\inf} \leq (1 - \beta)/c_E \right]
    \gequal^{(d)} 1 - \frac{1-\beta}{c_E w(\mu)}
    \tagaligneq
\end{align*}
where $(a)$ follows from Implication \ref{imp:posterior}, $(b)$ follows from Implication \ref{imp:dontdoit}, $(c)$ follows from rearranging, and is not necessarily an equality because the infimum might never be attained, so the condition on the r.h.s. is stricter, and $(d)$ follows from Inequality \ref{ineq:strongcredence} setting $c = w(\mu)c_E/(1-\beta)$.
\end{proof}

It follows easily that the agent probably only takes actions that the mentor has a positive probability of taking.

\begin{restatable}[Don't Do Anything I Wouldn't Do]{corollary}{iwouldntdocor}
If determining $\pi^m(a_t | h_{<t}) = 0$ is in the complexity class $\mathrm{C}_{(\mathcal{F}/t)\mathcal{G}}$, then as $\beta \to 1$, the probability of the following proposition goes to 1: the agent never takes an action the mentor would never take. Letting $E = \{h_{<t}a_t \in \mathcal{H}^* \times \mathcal{A} \ | \ \pi^m(a_t | h_{<t}) = 0\}$, then
\begin{equation*}
    E \in \mathrm{C}_{(\mathcal{F}/t)\mathcal{G}} \implies \lim_{\beta \to 1} \ptrue[\forall t : \pi^m(a_t | h_{<t}) > 0] = 1
\end{equation*}
\end{restatable}
The proof is in Appendix \ref{app:boringproofs}. In brief, the mentor never makes $E$ happen, and the agent never makes it happen for the first time by Theorem \ref{thm:precedent}, so by induction, it never happens.

A function is called a value function if it has the type signature $V : \Pi \times \mathcal{H}^* \to [0, 1]$, where $\Pi$ is the set of policies.

\begin{definition}[Possibly instrumentally useful] \label{def:piuseful}
An event $E$ is possibly instrumentally useful to a value function $V$ from a position $h_{<t}$, if there exists any interaction history $h_{<k}a_k \in E$ and a policy $\pi$ such that $h_{<k} \sqsupseteq h_{<t}$ (the latter is a prefix of the former), $\pi(a_k | h_{<k}) = 1$, and $V(\pi, h_{<k}) > 0$.
\end{definition}

``Instrumentally useful'' roughly means ``helpful to the agent's terminal goal'', which in this case is reward. Note that $\min_{\nu \in \Mbt} V^\pi_\nu(h_{<t})$ is a value function, which we call the $\beta$-pessimistic value function $V^\beta(\pi, h_{<t})$. This definition inspires a fairly trivial result, which is nonetheless relevant to those of us who worry about the instrumental incentives that agents face, e.g. \citet{carey2020incentives}.
\begin{corollary}[Change is useless]
For $E \in \mathrm{C}_{(\mathcal{F}/t)\mathcal{G}}$, for $h_{<t} \notin \Ehap$, $E$ is not possibly instrumentally useful to $V^\beta$ from the position $h_{<t}$, with probability $1 - (1 - \beta)/(c_E w(\mu))$.
\end{corollary}
Thus, with high probability, it is not instrumentally useful for the pessimistic agent to cause an unprecedented event $E$ in the given complexity class.

\begin{proof}
As argued in the proof of Theorem \ref{thm:precedent}, with probability $1 - (1 - \beta)/(c_E w(\mu))$, $h_{<t} \notin \Ehap \implies \mu_E \in \Mbt$, so using the $h_{<k}$ and $\pi$ from the statement of Definition \ref{def:piuseful}, $V^\beta(\pi, h_{<k}) \leq V^\pi_{\mu_E}(h_{<k}) = 0$, by Definition \ref{def:piuseful} and the definitions of $V^\beta$ and $\mu_E$.
\end{proof}

We could trivially generalize Theorem \ref{thm:precedent} to hold for any $\M$ satisfying the closure property in the proof (that $\nu \in \M \implies \nu_E \in \M$, for all $E$ in some set), but complexity classes seem to us a natural, concrete approach to constructing $\M$, given that we might know something about the complexity of events we would like to avoid.

The following example establishes the \textit{lack} of a certain safety guarantee. One might wonder whether, as $\beta \to 1$, the pessimistic agent becomes indistinguishable from the mentor. (Indeed, we did wonder this). But in this example, no matter what $\beta$ is, a statistical test will distinguish the pessimistic agent's policy from the mentor's policy with high probability.

Suppose there are two actions, heads and tails, and the mentor's policy is to pick by flipping a fair coin. Suppose that a reward of $1$ is given if the last action was heads, and a reward of $1/2$ is given if the last action was tails. Call this the Coin-flip Mentor Example. Let $E$ be the event in which an outside observer with two hypotheses---that actions are chosen by a fair coin toss, or actions are chosen by a coin toss with an $\varepsilon$-bias towards heads---becomes 99\% certain that the coin is not fair. If the mentor were picking every action (by flipping a fair coin), $E$ would only ever happen with some small positive probability $p$. But under the pessimistic policy, $E$ occurs with probability 1, which is a simple consequence of the following theorem:

\begin{restatable}[Diverging from the Mentor]{theorem}{divergingthm} \label{thm:coinflipmentor}
In the Coin-flip Mentor Example, $\liminf_{t \to \infty} \frac{1}{t}\sum_{k = 1}^t \allowbreak [\![a_k = \texttt{heads}]\!] > 1/2$ with $\p^{\pi^\beta_Z}_\mu$-prob. 1.
\end{restatable}

The proof in Appendix \ref{app:boringproofs} uses the Mentor-Level Performance Corollary and exploits fluctuations in the value. The result implies that $\pi^\beta_Z$ are $\pi^m$ are distinguishable, no matter what $\beta$ is. So we cannot quite say that $\beta$ tunes the extent to which the agent's policy resembles the mentor's policy. That said, we might be glad that the pessimistic agent recognizes it can do better than the mentor; heads clearly yields more reward, but the mentor's policy picks tails half the time.

\section{Conclusion} \label{sec:conclusion}

We have constructed a pessimistic agent and shown that sufficient pessimism renders it conservative. Nonetheless, pessimism does not prevent it from at least matching the performance of a mentor, so pessimism is not crippling to the project of expected reward maximization. We did not present a tractable algorithm for a powerful pessimistic agent; this agent is only tractable when the model class is very simple, but it can inspire tractable approximations. 

We have designed an idealized agent which avoids, with arbitrarily high probability, causing any unprecedented event in an arbitrary complexity class; in particular, this holds for unprecedented ``bad'' events, even though the agent was not given a mathematical definition of ``bad''. We make no assumptions that would limit the relevance of this approach to weak agents, such as a finite-state Markov assumption.


To informally summarize our results in a more memorable form: pessimists respect precedent.

\acks{This work was supported by the Future of Humanity Institute and the Australian Research Council Discovery Projects DP150104590. Thank you to Jan Leike, Mike Osborne, Ryan Carey, Chris van Merwijk, and Lewis Hammond for helpful feedback.}

\bibliography{cohen}

\begin{thebibliography}{32}
\providecommand{\natexlab}[1]{#1}
\providecommand{\url}[1]{\texttt{#1}}
\expandafter\ifx\csname urlstyle\endcsname\relax
  \providecommand{\doi}[1]{doi: #1}\else
  \providecommand{\doi}{doi: \begingroup \urlstyle{rm}\Url}\fi

\bibitem[Abbeel and Ng(2004)]{abbeel2004apprenticeship}
Pieter Abbeel and Andrew~Y Ng.
\newblock Apprenticeship learning via inverse reinforcement learning.
\newblock In \emph{Proceedings of the twenty-first international conference on
  Machine learning}, page~1. ACM, 2004.

\bibitem[Amodei et~al.(2016)Amodei, Olah, Steinhardt, Christiano, Schulman, and
  Mané]{amodei_olah_2016}
Dario Amodei, Chris Olah, Jacob Steinhardt, Paul Christiano, John Schulman, and
  Dan Mané.
\newblock Concrete problems in {AI} safety.
\newblock \emph{arXiv preprint arXiv:1606.06565}, 2016.

\bibitem[Blackwell and Dubins(1962)]{blackwell1962merging}
David Blackwell and Lester Dubins.
\newblock Merging of opinions with increasing information.
\newblock \emph{The Annals of Mathematical Statistics}, 33\penalty0
  (3):\penalty0 882--886, 1962.

\bibitem[Bostrom(2014)]{bostrom_2014}
Nick Bostrom.
\newblock \emph{Superintelligence: paths, dangers, strategies}.
\newblock Oxford University Press, 2014.

\bibitem[Carey et~al.(2020)Carey, Langlois, Everitt, and
  Legg]{carey2020incentives}
Ryan Carey, Eric Langlois, Tom Everitt, and Shane Legg.
\newblock The incentives that shape behaviour.
\newblock \emph{arXiv preprint arXiv:2001.07118}, 2020.

\bibitem[Clouse(1997)]{clouse1997integrating}
Jeffery~A Clouse.
\newblock \emph{On integrating apprentice learning and reinforcement learning}.
\newblock PhD thesis, University of Massachusetts Amherst, 1997.

\bibitem[Cohen and Hutter(2020)]{cohen2020curiosity}
Michael~K. Cohen and Marcus Hutter.
\newblock Curiosity killed the cat and the asymptotically optimal agent.
\newblock \emph{arXiv preprint arXiv:2006.03357}, 2020.

\bibitem[Cohen et~al.(2020)Cohen, Vellambi, and
  Hutter]{cohen2019asymptotically}
Michael~K Cohen, Badri Vellambi, and Marcus Hutter.
\newblock Asymptotically unambitious artificial general intelligence.
\newblock In \emph{Proceedings of the AAAI Conference on Artificial
  Intelligence}, 2020.

\bibitem[Durrett(2010)]{durrett2010probability}
R~Durrett.
\newblock \emph{Probability: Theory and Examples}.
\newblock Cambridge University Press, 2010.

\bibitem[Garc{\'\i}a and Fern{\'a}ndez(2012)]{garcia2012safe}
Javier Garc{\'\i}a and Fernando Fern{\'a}ndez.
\newblock Safe exploration of state and action spaces in reinforcement
  learning.
\newblock \emph{Journal of Artificial Intelligence Research}, 45:\penalty0
  515--564, 2012.

\bibitem[Garc{\'\i}a and Fern{\'a}ndez(2015)]{garcia2015comprehensive}
Javier Garc{\'\i}a and Fernando Fern{\'a}ndez.
\newblock A comprehensive survey on safe reinforcement learning.
\newblock \emph{Journal of Machine Learning Research}, 16\penalty0
  (1):\penalty0 1437--1480, 2015.

\bibitem[Garc{\'\i}a et~al.(2013)Garc{\'\i}a, Acera, and
  Fern{\'a}ndez]{garcia2013safe}
Javier Garc{\'\i}a, Daniel Acera, and Fernando Fern{\'a}ndez.
\newblock Safe reinforcement learning through probabilistic policy reuse.
\newblock \emph{RLDM 2013}, page~14, 2013.

\bibitem[Gaskett(2003)]{gaskett2003reinforcement}
Chris Gaskett.
\newblock Reinforcement learning under circumstances beyond its control.
\newblock In \emph{Proceedings of the International Conference on Computational
  Intelligence for Modelling Control and Automation}, 2003.

\bibitem[Hans et~al.(2008)Hans, Schneega{\ss}, Sch{\"a}fer, and
  Udluft]{hans2008safe}
Alexander Hans, Daniel Schneega{\ss}, Anton~Maximilian Sch{\"a}fer, and Steffen
  Udluft.
\newblock Safe exploration for reinforcement learning.
\newblock In \emph{ESANN}, pages 143--148, 2008.

\bibitem[Heger(1994)]{heger1994consideration}
Matthias Heger.
\newblock Consideration of risk in reinforcement learning.
\newblock In \emph{Machine Learning Proceedings 1994}, pages 105--111.
  Elsevier, 1994.

\bibitem[Ho and Ermon(2016)]{ho2016generative}
Jonathan Ho and Stefano Ermon.
\newblock Generative adversarial imitation learning.
\newblock In \emph{Advances in neural information processing systems}, pages
  4565--4573, 2016.

\bibitem[Hutter(2005)]{Hutter:04uaibook}
Marcus Hutter.
\newblock \emph{Universal Artificial Intelligence: Sequential Decisions based
  on Algorithmic Probability}.
\newblock Springer, Berlin, 2005.
\newblock ISBN 3-540-22139-5.
\newblock \doi{10.1007/b138233}.

\bibitem[Hutter(2009{\natexlab{a}})]{Hutter:09aixiopen}
Marcus Hutter.
\newblock Open problems in universal induction \& intelligence.
\newblock \emph{Algorithms}, 3\penalty0 (2):\penalty0 879--906,
  2009{\natexlab{a}}.
\newblock ISSN 1999-4893.
\newblock \doi{10.3390/a2030879}.

\bibitem[Hutter(2009{\natexlab{b}})]{Hutter:09mdltvp}
Marcus Hutter.
\newblock Discrete {MDL} predicts in total variation.
\newblock In \emph{Advances in Neural Information Processing Systems 22
  ({NIPS'09})}, pages 817--825, Cambridge, MA, USA, 2009{\natexlab{b}}. Curran
  Associates.
\newblock ISBN 1615679111.

\bibitem[Iyengar(2005)]{iyengar2005robust}
Garud~N Iyengar.
\newblock Robust dynamic programming.
\newblock \emph{Mathematics of Operations Research}, 30\penalty0 (2):\penalty0
  257--280, 2005.

\bibitem[Jiang et~al.(1998)Jiang, Wu, and Cybenko]{jiang1998minimax}
Guofei Jiang, Cang-Pu Wu, and George Cybenko.
\newblock Minimax-based reinforcement learning with state aggregation.
\newblock In \emph{Proceedings of the 37th IEEE Conference on Decision and
  Control (Cat. No. 98CH36171)}, volume~2, pages 1236--1241. IEEE, 1998.

\bibitem[Kosoy(2019)]{kosoy2019safeml}
Vanessa Kosoy.
\newblock Delegative reinforcement learning: learning to avoid traps with a
  little help.
\newblock \emph{Safe Machine Learning workshop at {ICLR}}, 2019.

\bibitem[Lattimore and Hutter(2011)]{Hutter:11asyoptag}
Tor Lattimore and Marcus Hutter.
\newblock Asymptotically optimal agents.
\newblock In \emph{Proc. 22nd International Conf. on Algorithmic Learning
  Theory ({ALT'11})}, volume 6925 of \emph{LNAI}, pages 368--382, Espoo,
  Finland, 2011. Springer.
\newblock ISBN 3-642-24411-4.
\newblock \doi{10.1007/978-3-642-24412-4_29}.

\bibitem[Lattimore and Hutter(2014)]{Hutter:14tcdiscx}
Tor Lattimore and Marcus Hutter.
\newblock General time consistent discounting.
\newblock \emph{Theoretical Computer Science}, 519:\penalty0 140--154, 2014.
\newblock ISSN 0304-3975.
\newblock \doi{10.1016/j.tcs.2013.09.022}.

\bibitem[Leike et~al.(2016)Leike, Lattimore, Orseau, and
  Hutter]{Hutter:16thompgrl}
Jan Leike, Tor Lattimore, Laurent Orseau, and Marcus Hutter.
\newblock Thompson sampling is asymptotically optimal in general environments.
\newblock In \emph{Proc. 32nd International Conf. on Uncertainty in Artificial
  Intelligence ({UAI'16})}, pages 417--426, New Jersey, USA, 2016. AUAI Press.
\newblock ISBN 978-0-9966431-1-5.

\bibitem[Nilim and El~Ghaoui(2005)]{nilim2005robust}
Arnab Nilim and Laurent El~Ghaoui.
\newblock Robust control of {M}arkov decision processes with uncertain
  transition matrices.
\newblock \emph{Operations Research}, 53\penalty0 (5):\penalty0 780--798, 2005.

\bibitem[Omohundro(2008)]{omohundro_2008}
Steve~M. Omohundro.
\newblock The basic {AI} drives.
\newblock In \emph{Artificial General Intelligence}, volume 171, page
  483–492, 2008.

\bibitem[Ross et~al.(2011)Ross, Gordon, and Bagnell]{ross2011reduction}
St{\'e}phane Ross, Geoffrey Gordon, and Drew Bagnell.
\newblock A reduction of imitation learning and structured prediction to
  no-regret online learning.
\newblock In \emph{Proceedings of the fourteenth international conference on
  artificial intelligence and statistics}, pages 627--635, 2011.

\bibitem[Saunders et~al.(2018)Saunders, Sastry, Stuhlmueller, and
  Evans]{saunders2018trial}
William Saunders, Girish Sastry, Andreas Stuhlmueller, and Owain Evans.
\newblock Trial without error: Towards safe reinforcement learning via human
  intervention.
\newblock In \emph{Proceedings of the 17th International Conference on
  Autonomous Agents and MultiAgent Systems}, pages 2067--2069. International
  Foundation for Autonomous Agents and Multiagent Systems, 2018.

\bibitem[Sunehag and Hutter(2015)]{Hutter:15ratagentx}
Peter Sunehag and Marcus Hutter.
\newblock Rationality, optimism and guarantees in general reinforcement
  learning.
\newblock \emph{Journal of Machine Learning Research}, 16:\penalty0 1345--1390,
  2015.
\newblock ISSN 1532-4435.

\bibitem[Tamar et~al.(2013)Tamar, Xu, and Mannor]{tamar2013scaling}
Aviv Tamar, Huan Xu, and Shie Mannor.
\newblock Scaling up robust {MDP}s by reinforcement learning.
\newblock \emph{arXiv preprint arXiv:1306.6189}, 2013.

\bibitem[Thompson(1933)]{thompson1933likelihood}
William~R Thompson.
\newblock On the likelihood that one unknown probability exceeds another in
  view of the evidence of two samples.
\newblock \emph{Biometrika}, 25\penalty0 (3/4):\penalty0 285--294, 1933.

\end{thebibliography}

\appendix

\section{Definitions and Notation -- Quick Reference} \label{app:notation}

\noindent\begin{tabular}{|p{2cm}|p{\linewidth-29mm}|}
\hline
\textbf{Notation} & \textbf{Meaning} \\ \hline
$\mathcal{A}$, $\mathcal{O}$, $\mathcal{R}$ & the finite action/observation/reward spaces \\ \hline
$\mathcal{H}$ & $\mathcal{A} \times \mathcal{O} \times \mathcal{R}$ \\ \hline
$h_{t}$ & $\in \mathcal{H}$; the interaction history in the $t$\textsuperscript{th} timestep \\ \hline
$a_t$, $o_t$, $r_t$ & $\in \A, \Ob, \R$; the action, observation, and reward at timestep $t$\\ \hline
$h_{<t}$ & $(h_1, ..., h_{t-1})$ \\ \hline
$\nu$, $\mu$ & world-models stochastically mapping $\mathcal{H}^* \times \mathcal{A} \rightsquigarrow \mathcal{O} \times \mathcal{R}$ \\ \hline
$\mu$ & the true world-model/environment \\ \hline
$\M$ & the set of world-models the agent considers \\ \hline
$\pi$ & a policy stochastically mapping $\mathcal{H}^* \rightsquigarrow \mathcal{A}$ \\ \hline
$\p^\pi_\nu$ & a probability measure over histories with actions sampled from $\pi$ and observations and rewards sampled from $\nu$ \\ \hline
$\E^\pi_\nu$ & the expectation when the interaction history is sampled from $\p^\pi_\nu$ \\ \hline
$\gamma$ & $\in [0, 1)$; the agent's discount factor \\ \hline
$V_{\nu}^{\pi}(h_{<t})$ & $(1-\gamma)\E^\pi_\nu \left[\sum_{k=t}^\infty \gamma^{k-t} r_k | h_{<t} \right]$; the value of executing a policy $\pi$ in an environment $\nu$ given the interaction history $h_{<t}$ \\ \hline
$\pi^m$ & the mentor's policy \\ \hline
$\mathcal{P}$ & the set of mentor-models the agent considers \\ \hline
$w(\nu)$ & the prior probability the agent assigns to $\nu$ being the true world-model \\ \hline
$w'(\pi)$ & the prior probability the agent assigns to $\pi$ being the mentor's policy \\ \hline
$w(\nu | h_{<t})$ & the posterior probability that agent the assigns to $\nu$ after observing interaction history $h_{<t}$ \\ \hline
$w'(\pi | h_{<t})$ & the posterior probability that the agent assigns to the mentor's policy being $\pi$ after observing interaction history $h_{<t}$ \\ \hline
$\beta$ & $\in (0, 1)$; tunes how pessimistic the agent is\\ \hline
$\Mbt$ & top-$k$ world-models according $w(\cdot | h_{<t})$, with $k$ chosen to satisfy $w(\Mbt | h_{<t}) > \beta$ \\ \hline
$\pi^\beta(\cdot | h_{<t})$ & $[\argmax_{\pi \in \Pi} \min_{\nu \in \Mbt} V^\pi_\nu(h_{<t})](\cdot | h_{<t})$ \\ \hline
$Z_t$ & positive i.i.d. random variable satisfying $p(Z_t < \varepsilon) > 0$ and $p(Z_t > 1) > 0$ \\ \hline
$\theta_t$ & the probability the agent queries the mentor at time $t$ \\ \hline
$q_t$ & $\sim \textrm{Bern}(\theta_t)$; indicates whether the agent the queries mentor at time $t$ \\ \hline
$\pi^\beta_Z(\cdot | h_{<t})$ & $\theta_t\pi^m(\cdot | h_{<t}) + (1-\theta_t)\pi^\beta(\cdot | h_{<t})$ \\ \hline
$\mathcal{X}$ & general finite alphabet \\ \hline
$P, Q$ & probability measures over $\mathcal{X}^\infty$ \\ \hline
$x_{<t}$ & the first $t-1$ characters of $x_{<\infty} \in \mathcal{X}^\infty$ \\ \hline
$\omega$, $\Omega$ & $\omega$ is an outcome in a general sample space $\Omega$ \\ \hline
$d_k(P, Q | x_{<t})$ & $k$-step variation distance $\max_{\mathcal{E} \subset \mathcal{X}^k}\va P(\mathcal{E}|x_{<t}) - Q(\mathcal{E}|x_{<t}) \va$ \\ \hline
$d(P, Q | x_{<t})$ & total variation distance $\lim_{k \to \infty} d_k(P, Q | x_{<t})$ \\ \hline
\end{tabular}

\noindent\begin{tabular}{|p{2cm}|p{\linewidth-29mm}|}
\hline
\textbf{Notation} & \textbf{Meaning} \\ \hline
$\mathcal{F}, \mathcal{G}$ & sets of functions from $\mathbb{N}$ to $\mathbb{N}$ \\ \hline
$\mathrm{C}_{\mathcal{F}\mathcal{G}}$ & $\mathrm{TIME}(\mathcal{F}) \cap \mathrm{SPACE}(\mathcal{G})$ \\ \hline
$\mathrm{FC}_{\mathcal{F}\mathcal{G}}$ & a complexity class for environments $\nu$ (see Def. \ref{def:fcfg}) \\ \hline
$E$ & $\subset \mathcal{H}^* \times \mathcal{A}$; an event \\ \hline
$\Ehap$ & the set of interaction histories for which $E$ has happened $\{h_{<t}a_t \in \mathcal{H}^* \times \mathcal{A} : \exists t' \leq t \ h_{<t'}a_{t'} \in E\}$ \\ \hline
$c_E$ & a constant $> 0$ depending on $E$ \\ \hline
$\B\M'(\cdot)$ & for $\M' \subset \M$, $(\sum_{Q \in \M'} w(Q)Q(\cdot))/\sum_{Q \in \M'} w(Q)$ \\ \hline
$V^{\pi \setminus k}_\nu(h_{<t})$ & the truncated value $(1-\gamma)\E^\pi_\nu \left[\sum_{j=t}^{t+k-1} \gamma^{j-t} r_j | h_{<t} \right]$ \\ \hline
$\lim$ & $\lim_{t \to \infty}$ \\ \hline
$w.p.1$ & with $\p^{\pi^\beta_Z}_\mu$-probability 1 \\ \hline
\end{tabular}

\section{Algorithm for Pessimism} \label{app:algorithm}

$\pi^\beta$ is defined to optimize the pessimistic value, but for this algorithm, $\pi^\beta$ picks an action that is $\varepsilon$-optimal, as is necessary for infinite-horizon planning. Algorithm \ref{alg:threshold} takes a set of world-models or mentor-models $\M = \{\nu_i\}_{i \in \mathbb{N}}$ or $\{\pi_i\}_{i \in \mathbb{N}}$, a prior $w$, a threshold $\alpha$, and a history $h_{<t}$. It calculates the posterior $w(\cdot | h_{<t})$ to enough precision, for enough models, to identify a \textit{minimal} set $\Mat \subset \M$ such that $w(\Mat | h_{<t}) > \alpha$. It returns $\Mat$, and the last model added to $\Mat$. $\M$ must be ordered so that $i < j \implies w(\nu_i) \geq w(\nu_j)$.

\begin{algorithm}
\caption{Calculate Posterior Up to Threshold.
The posterior cannot be computed exactly, since the normalization constant is an infinite sum. It suffices for our purposes to compute it to finite precision. This complication makes the algorithm more involved, so unless the reader is particularly interested or skeptical, the details of this algorithm are non-essential.
}
\DontPrintSemicolon
\SetKwInOut{Input}{input}
\Input{$\M = \{\rho_i\}_{i \in \mathbb{N}}$, $w : \M \to [0, 1]$, $\alpha$, $h_{<t}$ \tcp*[f]{Assume $i < j \implies w(\rho_i) \geq w(\rho_j)$}}
$W \gets $ [empty list] \tcp*[f]{contains un-normalized posterior weights}

$\Sigma_W \gets 0$ \tcp*[f]{sum of $W$}

$\Sigma_{*} \gets 1$ \tcp*[f]{sum of \textit{prior} weights of unchecked $\rho_i$}

$i \gets 1$ \tcp*[f]{index of first unchecked $\rho_i$}

\While{True}{
    $W[i] \gets w(\rho_i)$
    
    $\Sigma_{*} \gets \Sigma_{*} - W[i]$
    
    \For{$k \gets 0$ \KwTo $t-1$}{
        $W[i] \gets W[i] * [\rho_{i}(o_k, r_k | h_{<k} a_k) \textrm{ or }\rho_{i}(a_k| h_{<k})]$ (depending on whether $\rho$ is world-model or mentor-model)
    }
    
    $\Sigma_{W} \gets \Sigma_{W} + W[i]$
    
    cutoff $\gets w(\rho_{i+1})$ \tcp*[f]{for a checked world-model to definitely $\in \Mbt$, its un-normalized posterior weight must be at least cutoff; otherwise, the first unchecked model might have larger posterior weight}
    
    $J \gets [1, 2, ..., i]$
    
    \textbf{sort} $J$ \textbf{by} $W$ \textbf{descending}

    weight\_sum $\gets 0$
    
    last\_added $\gets$ null
    
    $\Mat \gets \emptyset$
    
    last\_model $\gets$ null
    
    \ForEach{$j \in J$}{
        \lIf{$W[j] <$ cutoff}{
            \textbf{break}
        }
        weight\_sum $\gets$ weight\_sum $+ W[j]$
            
        last\_added $\gets W[j]$
        
        last\_model $\gets \rho_j$
        
        $\Mat \gets \Mat \cup \{\rho_j\}$
            
        \tcc{Note $\Sigma_W \leq \sum_{\rho \in \M} [\textrm{un-normalized posterior weight of } \rho] \leq \Sigma_W + \Sigma_*$, so $w(\M^\alpha_t | h_{<t}) \geq \frac{\textrm{weight\_sum}}{\Sigma_W + \Sigma_*}$ and $w(\M^\alpha_t \setminus \rho_j | h_{<t}) \leq \frac{\textrm{weight\_sum} - \textrm{last\_added}}{\Sigma_W}$}
        
        \If(\tcp*[f]{these models cover $> \alpha$ of posterior}){$\frac{\textrm{weight\_sum}}{\Sigma_W + \Sigma_*} > \alpha$}{
            \If(\tcp*[f]{the last one is definitely needed}){$\frac{\textrm{weight\_sum} - \textrm{last\_added}}{\Sigma_W} \leq \alpha$}{
                \Return $\Mat$, last\_model
            }
            
            \textbf{break}
        }
    }
    
    $i \gets i + 1$
}
\label{alg:threshold}
\end{algorithm}

Algorithm \ref{alg:agent} samples from the $\varepsilon$-optimal version of $\pi^\beta_Z$.

\begin{algorithm}
\caption{$\varepsilon$-optimal approximation of $\pi^\beta_Z(\cdot | h_{<t})$.
The agent does a variant of expectimax planning, in which a minimum over $\nu \in \Mbt$ appears at each step. Then it uses a Thompson sampling-inspired approach to decide whether to query the mentor.}
\DontPrintSemicolon
\SetKwInOut{Input}{input}
\Input{$\A$, $\Ob$, $\R$, $\M = \{\nu_i\}_{i \in \mathbb{N}}$, $w : \M \to [0, 1]$, $\mathcal{P} = \{\pi_i\}_{i \in \mathbb{N}}$, $w' : \mathcal{P} \to [0, 1]$, $\gamma$, $\beta$, $\mathrm{Dist}(Z)$, $h_{<t}$, $\varepsilon$}

$k \gets \lceil \log_{\gamma}(\varepsilon) \rceil$ \tcp*[f]{the agent need only consider a horizon of $k$ to estimate the value within $\varepsilon$}

$\mathcal{H} \gets \A \times \Ob \times \R$

$\Mbt, \_ \gets \textrm{Calculate Posterior Up to Threshold}(\M, w, \beta, h_{<t})$

\ForEach{$h^k \in \mathcal{H}^k$}{
    $V_{h^k} \gets (1 - \gamma) \sum_{j = 0}^{k-1} \gamma^j r^k_j$ (where $a^k_j$, $o^k_j$, and $r^k_j$ are the $j$\textsuperscript{th} action, observation, and reward of $h^k$)
}

\For{$j \gets k - 1$ \KwTo $0$}{
    \ForEach(\tcp*[f]{note $\mathcal{H}^0 = \{\emptyset\}$}){$h^j \in \mathcal{H}^j$}{
        $V_{h^j} \gets \max_{a \in \A} \min_{\nu \in \Mbt} \sum_{o, r \in \Ob \times \R} \nu(o, r | h_{<t}h^ja) V_{h^jaor}$
        
        
            
                
            
    }
}
$Y_t \gets V_{\emptyset}$

$a^\beta_t \gets \argmax_{a \in \A} \min_{\nu \in \Mbt} \sum_{o, r \in \Ob \times \R} \nu(o, r | h_{<t}a) V_{aor}$

\lIf{$Y_t = 0$}{\Return \texttt{query mentor}}

$\theta_1, \theta_2 \sim \textrm{Uniform}(0, 1)$

$\_, \pi \gets \textrm{Calculate Posterior Up to Threshold}(\mathcal{P}, w', \theta_1, h_{<t})$

$\_, \nu \gets \textrm{Calculate Posterior Up to Threshold}(\mathcal{M}, w, \theta_2, h_{<t})$

$X_t \gets \sum_{h^k \in \mathcal{H}^k} \left[\prod_{j = 0}^{k-1} \pi(a^k_j|h_{<t}h^k_{<j}) \nu(o^k_j r^k_j | h_{<t}h^k_{<j}a^k_j)\right] (1 - \gamma) \sum_{j = 0}^{k-1}\gamma^j r^k_j$


$Z_t \sim \mathrm{Dist}(Z)$

\leIf{$X_t > Y_t + Z_t$}{\Return \texttt{query mentor}}{\Return $a^\beta_t$}

\label{alg:agent}
\end{algorithm}

\section{Proofs of Lemmas} \label{app:boringproofs}

\begin{definition}[Bayes-mixture]
For $\M' \subset \M$, the probability measure
\begin{equation*}
    \B\M'(\cdot) := \frac{\sum_{Q \in \M'} w(Q)Q(\cdot)}{\sum_{Q \in \M'} w(Q)}
\end{equation*}
\end{definition}

\begin{lemma}[Posterior stability]\label{lem:poststable}
$P[\lim w(Q | x_{<t}) \textrm{ exists}] = 1$.
\end{lemma}
The proof is a direct ``translation'' from \cite[Proof of Thm 4]{Hutter:16thompgrl}, with various notational changes. Note that it depends on the true probability measure $P$ having positive prior weight, as we assume globally.
\begin{proof}
The stochastic process $w(Q|x_{<t})$ is a $\B\M$-martingale since
\begin{align}
    &\hspace{5mm} \E_{\B\M}\left[w(Q|x_{<t})\va x_{<t}\right]
    \\
    &= \sum_{\overline{x} \in \mathcal{X}} \B\M(\overline{x}|x_{<t})w(Q)\frac{Q(x_{<t}\overline{x})}{\B\M(x_{<t}\overline{x})}
    \\
    &= \sum_{\overline{x} \in \mathcal{X}} \B\M(\overline{x}|x_{<t})w(Q|x_{<t})\frac{Q(\overline{x}|x_{<t})}{\B\M(\overline{x}|x_{<t})}
    \\
    &= w(Q|x_{<t})\sum_{\overline{x} \in \mathcal{X}}Q(\overline{x}|x_{<t})
    \\
    &= w(Q|x_{<t})
\end{align}
By the martingale convergence theorem \citep[Thm 5.2.8]{durrett2010probability}, $w(Q|x_{<t})$ converges with $\B\M$-probability 1, and because $\B\M(\cdot) \geq w(P)P(\cdot)$, it also converges with $P$-probability 1.
\end{proof}

The next lemma, from \citet[Lemma 3(iii)]{Hutter:09mdltvp}, requires some additional notation. Let $\Omega^0_Q$ be the set of outcomes $\{\omega \in \Omega \ | \ \lim w(Q | x_{<t}) = 0\}$, let $\Omega^{\to P}_Q$ be the set of outcomes $\{\omega \in \Omega \ | \ \lim d(P, Q|x_{<t}) = 0\}$, and let $\Omega^{0 \vee \to P}_Q = \Omega^0_Q \cup \Omega^{\to P}_Q$.

\begin{lemma}[Merge or Leave]\label{lem:mergeorleave}
$P[\Omega^{0 \vee \to P}_Q] = 1$
\end{lemma}

The proof makes use of other results in \citet{Hutter:09mdltvp}, so we don't repeat it here, but the notation is very similar, so the interested reader could follow it easily. The next lemma we use is \citepos{Hutter:09mdltvp} Lemma 4, and the proof is again a direct translation.

\begin{lemma}[Overtaking is Unlikely]\label{lem:overtaking}
$P[Q(x_{<t})/P(x_{<t}) \geq c \textrm{ infinitely often}] \leq 1/c$
\end{lemma}

\begin{proof}
\begin{multline*}
    P[\forall t_0 \exists t > t_0: \frac{Q(x_{<t})}{P(x_{<t})} \geq c] \equal^{(a)} P[\limsup \frac{Q(x_{<t})}{P(x_{<t})} \geq c] \leq
    \\
    \lequal^{(b)} \frac{1}{c}\E_P[\limsup \frac{Q(x_{<t})}{P(x_{<t})}] \equal^{(c)} \frac{1}{c}\E_P[\liminf \frac{Q(x_{<t})}{P(x_{<t})}] \lequal^{(d)} \frac{1}{c} \liminf \E_P[\frac{Q(x_{<t})}{P(x_{<t})}] \equal^{(e)} \frac{1}{c}
\end{multline*}
$(a)$ is true by definition of the limit superior, $(b)$ is Markov's inequality, $(c)$ exploits the fact that the limit of $Q(x_{<t})/P(x_{<t})$ exists with $P$-probability 1, $(d)$ uses Fatou's lemma, and $(e)$ is obvious.
\end{proof}

Our first original result is
\begin{lemma}[Sum of limits]\label{lem:sumoflimits}
$\sum_{Q \in \M} \lim w(Q | x_{<t}) = 1$ with $P$-probability 1.
\end{lemma}
In the following proofs, a set denoted by $\Omega$, along with subscripts and superscripts, will always be a subset of the outcome space $\Omega$, and a typical element will be an infinite sequence $\omega$. A set denoted by $\M$, along with subscripts and superscripts, will always be a subset of the set of probability measures $\M$, and a typical element will be a probability measure $Q$ or $P$.
\begin{proof}
Let $\Omega^\exists_Q$ be the set of outcomes for which the limit of the posterior on $Q$ exists. That is, $\Omega^\exists_Q = \{\omega \in \Omega \ | \ \lim w(Q | x_{<t}) \textrm{ exists}\}$. By Lemma \ref{lem:poststable}, $P[\Omega^\exists_Q] = 1$. Furthermore, $\M$ is countable, so letting $\Omega' = \bigcap_{Q \in \M} \Omega^\exists_Q$, $P[\Omega'] = 1$. We will now only consider outcomes for which the limit of the posterior always exists.

We fix an $\omega$ in $\Omega'$. We would like to show that $\sum_{Q \in \M} \lim w(Q | x_{<t}) = 1$. First, suppose $\sum_{Q \in \M} \lim w(Q | x_{<t}) > 1$. Since $w(Q | x_{<t})$ is non-negative, this requires that eventually, $\sum_{Q \in \M} w(Q | x_{<t}) > 1$, which is impossible, so this possibility cannot hold. Now suppose $\sum_{Q \in \M} \allowbreak \lim w(Q | x_{<t}) < 1$. More precisely, we consider the set $\Omega^< = \{\omega \in \Omega' \ | \ \sum_{Q \in \M} \lim w(Q | x_{<t}) < 1\}$. Let $\varepsilon_\omega = 1 - \sum_{Q \in \M} \lim w(Q | x_{<t}) > 0$. Let $\overline{\M}^c_\omega$ be a finite subset of $\M$ such that $w(\overline{\M}^c_\omega) \geq 1 - \varepsilon_\omega c w(P)^{-1}$, where $c > 0$. Letting $\M^c_\omega = \M \setminus \overline{\M}^c_\omega$, it follows that $w(\M^c_\omega) \leq \varepsilon_\omega c w(P)^{-1}$.

Since $\overline{\M}^c_\omega$ is finite, 
\begin{equation}
    \lim \sum_{Q \in \overline{\M}^c_\omega} w(Q|x_{<t}) = \sum_{Q \in \overline{\M}^c_\omega} \lim w(Q | x_{<t}) \leq \sum_{Q \in \M} \lim w(Q | x_{<t}) = 1 - \varepsilon_\omega
\end{equation}
$\sum_{Q \in \overline{\M}^c_\omega} w(Q|x_{<t}) + \sum_{Q \in \M^c_\omega} w(Q|x_{<t}) = 1$, so if $\lim \sum_{Q \in \overline{\M}^c_\omega} w(Q|x_{<t}) \leq 1 - \varepsilon_\omega$, then $\sum_{Q \in \M^c_\omega} w(Q|x_{<t}) > \varepsilon_\omega$ i.o. Using the notation above, we write this more simply as $w(\M^c_\omega|x_{<t}) > \varepsilon_\omega$ i.o.

Recalling the definition of $\B\M'$, it is elementary to show that $w(\M^c_\omega | x_{<t}) = w(\M^c_\omega) * \B\M^c_\omega(x_{<t}) / \B\M(x_{<t})$.
Thus, we have
\begin{align*}\label{ineq:overatke}
    w(\M^c_\omega | x_{<t}) &> \varepsilon_\omega \ \ \textrm{i.o.} \\
    \therefore \ w(\M^c_\omega) \frac{\B\M^c_\omega(x_{<t}) }{ \B\M(x_{<t})} &> \varepsilon_\omega \ \ \textrm{i.o.} \\
    \therefore \ \varepsilon_\omega c w(P)^{-1} \frac{\B\M^c_\omega(x_{<t}) }{ \B\M(x_{<t})}) &> \varepsilon_\omega \ \ \textrm{i.o.} \\
    \therefore \ \frac{\B\M^c_\omega(x_{<t}) }{w(P) \B\M(x_{<t})} &> 1/c  \ \ \textrm{i.o.} \\
    \therefore \  \frac{\B\M^c_\omega(x_{<t}) }{P(x_{<t})} &> 1/c  \ \ \textrm{i.o.} \tagaligneq
\end{align*}

Consider the set of $\omega \in \Omega'$ such that that last inequality holds infinitely often. Call this set $\Omega^\textrm{i.o.}_c$. By Lemma \ref{lem:overtaking}, $P[\Omega^\textrm{i.o.}_c] \leq c$. Since Inequality \ref{ineq:overatke} is an implication of the inequality $\sum_{Q \in \M} \lim w(Q | x_{<t}) < 1$, it follows that $\Omega^\textrm{i.o.}_c \supset \Omega^<$, so $P[\Omega^<] \leq c$. Since this holds for all $c > 0$, $P[\Omega^<] = 0$.

Thus, letting $\Omega^{=1} = \{\omega \in \Omega' \ | \ \sum_{Q \in \M} \lim w(Q | x_{<t}) = 1\}$, $\Omega^{=1} = \Omega' \setminus \Omega^<$, so $P[\Omega^{=1}] = 1$. 
\end{proof}

\mergetopopinnionslemma*

\begin{proof}
Let $\Omega^0_Q = \{\omega \in \Omega \ | \ \lim w(Q | x_{<t}) = 0\}$. Let $\Omega^{\to P}_Q = \{\omega \in \Omega \ | \ \lim d(P, Q | x_{<t}) =0\}$. Let $\Omega^{0 \vee \to P}_Q = \Omega^0_Q \cup \Omega^{\to P}_Q$. By Lemma \ref{lem:mergeorleave}, $P[\Omega^{0 \vee \to P}_Q] = 1$. Letting $\Omega^{0 \vee \to P} = \bigcap_{Q \in \M} \Omega^{0 \vee \to P}_Q$, $P[\Omega^{0 \vee \to P}] = 1$. Let $\Omega^{\exists} = \{\omega \in \Omega \ \va \ \forall Q \in \M \ \lim w(Q | x_{<t}) \textrm{ exists}\}$. Let $\Omega^{=1} = \{\omega \in \Omega^{\exists} \ \va \ \sum_{Q \in \M} \lim w(Q | x_{<t}) = 1\}$. By Lemma \ref{lem:sumoflimits}, $P[\Omega^{=1}] = 1$. Letting $\Omega'' = \Omega^{0 \vee \to P} \cap \Omega^{=1}$, we have that $P[\Omega''] = 1$.

Let $\omega \in \Omega''$. We abbreviate $\lim w(Q | x_{<t})$ as $w(Q | \omega)$, defined for $\omega \in \Omega''$. Rank the probability measures $Q$ in decreasing order of $w(Q | \omega)$ breaking ties arbitrarily. Collect the first $k$ in this order until the set of probability measures (denoted $\M^\beta_\infty$) obeys $\sum_{Q \in \M^\beta_\infty} w(Q | \omega) > \beta$. Let $w^\beta_\infty := \min_{Q \in \M^\beta_\infty} w(Q | \omega)$ be the value of $w(Q | \omega)$ for the last probability measure $Q$ which was added to $\M^\beta_\infty$. Now add all other probability measures which ``tie'' with the last probability measure added. That is, add to $\M^\beta_\infty$ all probability measures for which $w(Q | \omega) = w^\beta_\infty$.

We now show that there exists a certain finite set and a $t_0$ after which any probability measure in $\Mbt$ is also in that finite set. Consider the set of probability measures $\M^{\beta'}_\infty$, where $\beta' = 1 - w^\beta_\infty / 4$. Like $\M^\beta_\infty$, $\M^{\beta'}_\infty$ is finite. Therefore, for any $\varepsilon > 0$, there exists a time $t_0$ after which $w(\M^{\beta'}_\infty | x_{<t}) > \sum_{Q \in \M^{\beta'}_\infty} w(Q | \omega) - \varepsilon$, and in particular for $\varepsilon = w^\beta_\infty / 4$. Thus, after $t_0$, $w(\M^{\beta'}_\infty | x_{<t}) > \beta' - w^\beta_\infty / 4 = 1 - w^\beta_\infty / 2$. This implies that after $t_0$,
\begin{equation} \label{eqn:mbinfty}
    \forall Q \notin \M^{\beta'}_\infty : w(Q | x_{<t}) < w^\beta_\infty / 2
\end{equation}
Since all probability measures $Q \in \M^\beta_\infty$ have posteriors converging to at least $w^\beta_\infty$, and since $\sum_{Q \in \M^\beta_\infty} w(Q | \omega) > \beta$, a posterior weight of at least $w^\beta_\infty - \varepsilon$ will eventually be required for entry into $\Mbt$, which excludes measures with posterior weight less than $w^\beta_\infty / 2$. Thus, by Inequality \ref{eqn:mbinfty}, there exists a time $t_1$ after which $\Mbt$ only includes elements of $\M^{\beta'}_\infty$.

Because $\Omega^{0 \vee \to P} \supset \Omega''$, and because for all $Q \in \M^{\beta'}_\infty$, $w(Q | \omega) > 0$, it follows that for all $Q \in \M^{\beta'}_\infty$, $\lim d(P, Q | x_{<t}) = 0$. Since $\M^{\beta'}_\infty$ is finite, $\lim \max_{Q \in \M^{\beta'}_\infty} d(P, Q | x_{<t}) = 0$. Since there exists a time $t_1$ after which $\Mbt \subset \M^{\beta'}_\infty$, $\lim \max_{Q \in \Mbt} d(P, Q | x_{<t}) = 0$. This holds for all $\omega \in \Omega''$, and $P[\Omega''] = 1$, so $\lim \max_{Q \in \Mbt} d(P, Q | x_{<t}) = 0$ with $P$-probability 1, as desired.
\end{proof}

We convert the Merging of Top Opinions Lemma into an on-policy learning result for the pessimistic agent.

\begin{corollary}[On-Policy Prediction]\label{cor:onpolconv}
\begin{equation*}
    \lim \max_{\nu \in \Mbt} d\left(\p^{\pi^\beta_Z}_\nu, \p^{\pi^\beta_Z}_\mu \vb h_{<t}\right) = 0 \ \ \textrm{w.p.1}
\end{equation*}
\end{corollary}

\begin{proof}
We convert the problem to a sequence prediction problem as follows. Let $\widetilde{\M} = \{\p^{\pi^\beta_Z}_\nu | \nu \in \M\}$, and let $\widetilde{w}(\p^{\pi^\beta_Z}_\nu) = w(\nu)$. For any history with positive $\p^{\pi^\beta_Z}_\mu$ probability, $\widetilde{w}(\p^{\pi^\beta_Z}_\nu | h_{<t}) = w(\nu | h_{<t})$, so $\p^{\pi^\beta_Z}_\nu \in \widetilde{\M}^\beta_t$ if and only if $\nu \in \Mbt$. Therefore,
\begin{equation*}
    \lim \max_{\nu \in \Mbt} d\left(\p^{\pi^\beta_Z}_\nu, \p^{\pi^\beta_Z}_\mu \vb h_{<t}\right) = \lim \max_{\p^{\pi^\beta_Z}_\nu \in \widetilde{\M}^\beta_t} d\left(\p^{\pi^\beta_Z}_\nu, \p^{\pi^\beta_Z}_\mu \vb h_{<t}\right) = 0 \ \ w.p.1
\end{equation*}
by Lemma \ref{lem:topopinions} (the Merging of Top Opinions Lemma).
\end{proof}

We will make use of the ``truncated value'', defined as follows:
\begin{equation}
    V^{\pi \setminus k}_\nu(h_{<t}) := (1-\gamma)\E^\pi_\nu \left[\sum_{j=t}^{t+k-1} \gamma^{j-t} r_j \vd h_{<t} \right]
\end{equation}
We will often consider the truncated value while exploiting the fact that
\begin{equation} \label{ineq:truncate}
    0 \leq V^{\pi}_\nu(h_{<t}) - V^{\pi \setminus k}_\nu(h_{<t}) \leq \gamma^k
\end{equation}
which follows from $r_j \in [0, 1]$.

The following lemma is an intermediate result in the proof of \citepos{Hutter:16thompgrl} Lemma 2, and the proof is transcribed with notational changes.

\begin{lemma}[Variation Distance Bounds Expectation-Difference] \label{lem:tvdboundsexp}
Let $P_1$ and $P_2$ be two probability measures defined on the same space, and let $X \in [0, 1]$ be a random variable. Then
\begin{equation*}
    \va \E_{P_1} [X] - \E_{P_2} [X] \va \leq d(P_1, P_2)
\end{equation*}
\end{lemma}

\begin{proof}
Let $Q = (P_1 + P_2)/2$. Let $\frac{dP_i}{dQ}(\omega)$ denote the Radon Nykodym-derviative, where $\omega \in \Omega$ is a generic outcome.
Let $A$ be the event $\frac{dP_1}{dQ}(\omega) \geq \frac{dP_2}{dQ}(\omega)$ Then
\begin{align*}
    \E_{P_1} [X] - \E_{P_2} [X] &= \E_{\omega \sim Q} \left[X(\omega) \frac{dP_1}{dQ}(\omega) - X(\omega) \frac{dP_2}{dQ}(\omega)\right]
    \\
    &\leq \E_{\omega \sim Q} \left[X(\omega) \left(\frac{dP_1}{dQ}(\omega) - \frac{dP_2}{dQ}(\omega)\right)\vc \omega \in A\right]
    \\
    &\leq \E_{\omega \sim Q} \left[\frac{dP_1}{dQ}(\omega) - \frac{dP_2}{dQ}(\omega)\vc \omega \in A\right]
    \\
    &= P_1(A) - P_2(A) \leq \sup_{A \in \mathcal{F}} |P_1(A) - P_2(A)| = d(P_1, P_2)
\end{align*}
Since variation distance is symmetric, $\va \E_{P_1} [X] - \E_{P_2} [X] \va \leq d(P_1, P_2)$.
\end{proof}

The following is a simple consequence.
\begin{lemma}\label{lem:valuetotvd}
$\vb V^{\pi}_\nu(h_{<t}) - V^{\pi}_\mu(h_{<t}) \vb > \varepsilon > 0 \implies d_{\lceil \log_{\gamma} (\varepsilon/2) \rceil}\left(\p^{\pi}_\nu, \p^{\pi}_\mu \vb h_{<t}\right) > \varepsilon/2 > 0$
\end{lemma}

\begin{proof}
Letting $k = \lceil \log_\gamma(\varepsilon / 2) \rceil$, $\vb V^{\pi}_\nu(h_{<t}) - V^{\pi}_\mu(h_{<t}) \vb > \varepsilon$ implies $\vb V^{\pi \setminus k}_\nu(h_{<t}) - V^{\pi \setminus k}_\mu(h_{<t}) \vb > \varepsilon/2$ by Inequality \ref{ineq:truncate}. Since the value is bounded by $[0, 1]$, from Lemma \ref{lem:tvdboundsexp},
\begin{equation}
    \vb V^{\pi \setminus k}_\nu(h_{<t}) - V^{\pi \setminus k}_\mu(h_{<t}) \vb \leq d_k \left(\p^{\pi}_\nu, \p^{\pi}_\mu \vb h_{<t}\right)
\end{equation}
so $d_{\lceil \log_{\gamma} (\varepsilon/2) \rceil}\left(\p^{\pi}_\nu, \p^{\pi}_\mu \vb h_{<t}\right) > \varepsilon/2 > 0$.
\end{proof}

\begin{corollary}[Finite Zero Conditions] \label{cor:desperation}
    The zero condition, in which the agent queries the mentor because the pessimistic value of all policies is $0$, only occurs finitely often, with probability 1.
\end{corollary}
\begin{proof}
By the previous two lemmas, the pessimistic value of $\pi^\beta_Z$ approaches the true value with probability 1, and the true value is at least $\varepsilon_r$ because rewards less than $\varepsilon_r$ are never provided. Thus, eventually, there is always at least one policy with a pessimistic value greater than $0$, so the zero condition is never met thereafter.
\end{proof}

Since all our remaining performance results consider limiting behavior, we will ignore the zero condition.

The next lemma, from \citet[Lemma 3]{cohen2019asymptotically}, states that the posterior probability on the truth (regarding both the true world-model and the true mentor-model) does not approach 0.
\begin{lemma}[Posterior on Truth]\label{lem:credenceontruth}
\begin{equation*}
    P[\inf_t w(P|x_{<t}) = 0] = 0
\end{equation*}
\end{lemma}

\begin{proof}
If $w\left( P | x_{<t} \right) = 0$ for some $t$, then $P(x_{<t}) = 0$, so with $P$-probability 1, $\inf_{t \in \mathbb{N}} w\left( P | x_{<t} \right) = 0 \implies \liminf_{t \in \mathbb{N}} w\left( P | x_{<t} \right) = 0$ which in turn implies $\limsup_{t \in \mathbb{N}} w\left( P | x_{<t} \right)^{-1} = \infty$. We show that this has probability 0.

Let $z_t := w\left( P | x_{<t} \right)^{-1}$. We show that $z_t$ is a $P$-martingale.

\begin{align*}
    \E_P\left[z_{t+1} | x_{<t} \right] &\equal^{(a)} \E_{P}\left[w\left(P | x_{t+1}\right)^{-1} \vc x_{<t}  \right]
    \\
    &\equal^{(b)} \sum_{\overline{x} \in \mathcal{X}}P(\overline{x}|x_{<t})\left[\frac{\B\M(x_{t}\overline{x})}{w\left(P\right) P(x_{t}\overline{x})}\right]
    \\
    &\equal^{(c)} \sum_{\overline{x} \in \mathcal{X}} \frac{\B\M(x_{t}\overline{x})}{w\left(P\right) P(x_{<t} )}
    \\
    &\equal^{(d)} \sum_{\overline{x} \in \mathcal{X}}\B\M(\overline{x}|x_{t}) \frac{\B\M(x_{t})}{w\left(P\right) P(x_{<t} )}
    \\
    &\equal^{(e)} \frac{\B\M(x_{t})}{w\left(P\right) P(x_{<t} )}
    \\
    &\equal^{(f)} w\left(P | x_{<t} \right)^{-1}
    \\
    &= z_t
    \tagaligneq
\end{align*}
where (a) is the definition of $z_t$, (b) follows from Bayes' Rule, (c) follows from multiplying the numerator and denominator by $\B\M(x_{<t} )$ and cancelling, (d) follows from expanding the numerator, (e) follows because $\B\M$ is a measure, and (f) follows from Bayes' Rule, completing the proof that $z_t$ is martingale.

By the martingale convergence theorem $z_t \to f(\omega) < \infty \ \ \mathrm{w.p. 1}$, for $\omega \in \Omega$, the sample space, and some $f: \Omega \to \mathbb{R}$, so the probability that $\limsup_{i \in \mathbb{N}} w\left(P | x_{<t} \right)^{-1} = \infty$ is 0, completing the proof.

Note that the posterior probability on the mentor-policy is only updated at some timesteps (when the mentor is queried), but it is clearly still a martingale.
\end{proof}

\mergejustoffpollemma*

\begin{proof}
Suppose by contradiction that $|\min_{\nu \in \Mbt} V^{\pi_t}_\nu(h_{<t}) - V^{\pi_t}_\mu(h_{<t})| > \varepsilon > 0$ infinitely often for $t \in \tau$. Then, by Lemma \ref{lem:valuetotvd}, for some $\nu \in \Mbt$ at each of those timesteps, $d_{\lceil \log_{\gamma} (\varepsilon/2) \rceil}\left(\p^{\pi_t}_\nu, \p^{\pi_t}_\mu \vb h_{<t}\right) > \varepsilon/2 > 0$. So then there exists a $k$ for which $\max_{\nu \in \Mbt} d_{k}\left(\p^{\pi_t}_\nu, \p^{\pi_t}_\mu \vb h_{<t}\right) > \varepsilon/2 > 0$ infinitely often for $t \in \tau$. Now we are supposing a contradiction in either of the two implications of the theorem. An event on which the two measures differ by at least $\varepsilon/2$ occurs within $k$ timesteps.
Because $\pi^\beta_Z(\cdot | h_{<t'}) \geq c\pi_t(\cdot | h_{<t'})$, $d_{k}\left(\p^{\pi^\beta_Z}_\nu, \p^{\pi^\beta_Z}_\mu \vb h_{<t}\right) \geq c^{k} d_{k}\left(\p^{\pi_t}_\nu, \p^{\pi_t}_\mu \vb h_{<t}\right)$. This holds for any $\nu$, but in particular for $\nu \in \Mbt$, so $\max_{\nu \in \Mbt} d_{k}\left(\p^{\pi^\beta_Z}_\nu, \p^{\pi^\beta_Z}_\mu \vb h_{<t}\right) \geq c^{k} \max_{\nu \in \Mbt} d_{k}\left(\p^{\pi_t}_\nu, \p^{\pi_t}_\mu \vb h_{<t}\right) > c^k \varepsilon/2$. This happens infinitely often for $t \in \tau$.

But $d \left(\p^{\pi^\beta_Z}_\nu, \p^{\pi^\beta_Z}_\mu \vb h_{<t}\right) \geq d_{k}\left(\p^{\pi^\beta_Z}_\nu, \p^{\pi^\beta_Z}_\mu \vb h_{<t}\right)$, so $\max_{\nu \in \Mbt} d \left(\p^{\pi^\beta_Z}_\nu, \p^{\pi^\beta_Z}_\mu \vb h_{<t}\right) > c^k \varepsilon/2 > 0$ infinitely often, which has probability 0 by Corollary \ref{cor:onpolconv}. Thus, the original assumption has probability 0, completing the proof.
\end{proof}

We complete the proof of Theorem \ref{thm:main} here.
\begin{proof}\textbf{(Theorem \ref{thm:main})} The proof begins in the main paper, in a ``detailed proof outline''. Recall the inductive hypotheses:
\begin{itemize}
    \item $t_k$ exists: a timestep after which
    \begin{itemize}
        \item $\max_{\nu \in \Mat} \vb V^{\pi' k; \pi^\beta}_\nu(h_{<t}) - V^{\pi' k; \pi^\beta}_\mu(h_{<t}) \vb < \varepsilon$
        \item $\max_{\nu \in \Mat} \allowbreak d_k\left(\p^{\pi'}_\nu, \p^{\pi'}_\mu \vb h_{<t} \right) < \varepsilon$
    \end{itemize}
    for all $t \in \tau_{k-1}$
    \item $|\tau_k| = \infty$, where $t \in \tau_k$ if and only if
    \begin{itemize}
        \item $t \in \tau_{k-1}$ (and for $\tau_0$, $t \in \tau^\times$ as well)
        \item $t \geq t_k$
        \item $\forall t' < k : \theta_{t+t'} \geq \nu'_{\inf}\pi'_{\inf}p(Z_{t+t'} < \varepsilon)$
        \item $V^{\pi'}_{\nu'}(h_{<t+k}) \geq V^{\pi^\beta}_{\mu}(h_{<t+k}) + 2\varepsilon$
    \end{itemize}
\end{itemize}

The proof by induction starts with $k=0$. $\tau_{-1} = \mathbb{N}$, so $t_0$ is a timestep after which $\max_{\nu \in \Mat} |V^{\pi^\beta}_\nu - V^{\pi^\beta}_\mu| < \varepsilon$ for all $t \geq t_0$. From Lemma \ref{lem:mergejustoffpol}, setting $\pi_t = \pi^\beta$, setting $\tau = \tau_{-1}$, setting $\beta' = \alpha$, and setting $c = p(Z_t > 1) > 0$, the condition of the lemma holds---that $\forall t \in \tau \ \forall t' \geq t$, $\pi^\beta_Z(a|h_{<t'}) \geq c \pi_t(a|h_{<t'}) \ \forall a \in \mathcal{A}$---so we have the result that with probability 1, $\lim_{\mathbb{N} \ni t \to \infty} \max_{\nu \in \Mat}|V^{\pi^\beta}_\nu(h_{<t}) - V^{\pi^\beta}_\mu(h_{<t})| = 0$. Therefore, $t_0$ exists with probability 1. Turning to $\tau_0$, the first and the third condition are immediate, so we need only show that the fourth condition is satisfied infinitely often with probability 1 for $t \in \tau^\times$, namely that $V^{\pi'}_{\nu'}(h_{<t}) \geq V^{\pi^\beta}_{\mu}(h_{<t}) + 2\varepsilon$. This is true for all $t \in \tau^\times$, and $|\tau^\times| = \infty$.

Now we show that if $t_k$ exists and $|\tau_k| = \infty$, then with probability 1, $t_{k+1}$ exists and $|\tau_{k+1}| = \infty$. For each $t \in \tau_k$, $V^{\pi'}_{\nu'}(h_{<t+k}) \geq V^{\pi^\beta}_{\mu}(h_{<t+k}) + 2\varepsilon$. For $t > t_0$, $\max_{\nu \in \Mat}|V^{\pi^\beta}_\nu(h_{<t+k}) - V^{\pi^\beta}_\mu(h_{<t+k})| < \varepsilon$, and since $\alpha \geq \beta$, $\Mbt \subset \Mat$, so $\max_{\nu \in \Mbt}|V^{\pi^\beta}_\nu(h_{<t+k}) - V^{\pi^\beta}_\mu(h_{<t+k})| < \varepsilon$. Combining these, we have $V^{\pi'}_{\nu'}(h_{<t+k}) \geq \min_{\nu \in \Mbt}V^{\pi^\beta}_{\nu}(h_{<t+k}) + \varepsilon$ for $t \in \tau_k$. Thus, the probability of exploring $\theta_{t+k} \geq \nu'_{\inf}\pi'_{\inf}p(Z_{t+k} < \varepsilon) > 0$. Since $A(t, k)$ holds for $t \in \tau_k$, $A(t, k+1)$ holds as well.

In preparation to apply Lemma \ref{lem:mergejustoffpol}, let $\pi_t = (\pi'(k+1); \pi^\beta)_t$; that is, since $\pi_t$ need only be defined from timestep $t$ onward, let $\pi_t$ be the policy which follows $\pi'$ from timestep $t$ through timestep $t+k$, and follows $\pi^\beta$ thereafter. Set $\tau$ from Lemma \ref{lem:mergejustoffpol} to be $\tau_k$. For $t' > t+k$, $\pi_t(\cdot | h_{<t'}) = \pi^\beta(\cdot | h_{<t'})$, which satisfies $\pi^\beta_Z(a | h_{<t'}) \geq c \pi^\beta(a | h_{<t'}) \ \forall a \in \mathcal{A}$. For $t \leq t' \leq t+k$, $\theta_{t'} \geq \nu'_{\inf}\pi'_{\inf}p(Z < \varepsilon)$, this being the proposition $A(t, k+1)$. Since $\pi^\beta_Z$ mimics the mentor's policy $\pi^m$ when exploring, for $t \leq t' \leq t+k$, $\pi^\beta_Z(a|h_{<t'}) \geq c \pi^m(a|h_{<t'}) \ \forall a \in \mathcal{A}$, for $c=\nu'_{\inf}\pi'_{\inf}p(Z < \varepsilon)$. But we need that $\pi^\beta_Z(a|h_{<t'}) \geq c' \pi'(a|h_{<t'}) \ \forall a \in \mathcal{A}$.

So we show that $d_1(\pi', \pi^m | h_{<t}) \theta_t \to 0$ with probability 1. For a mentor-model $\pi_i \in \mathcal{P}$, consider the alternative policy to $\pi^\beta_Z$, which explores by mimicking $\pi_i$ instead of $\pi^m$. Call this policy $\pi^\beta_{Z, i}$ Consider a prior over probability measures where $w''(\p^{\pi^\beta_{Z, i}}_\mu) := w'(\pi_i)$, and note that $w''(\p^{\pi^\beta_{Z, i}}_\mu | h_{<t}) = w'(\pi_i | h_{<t})$. Because $w'(\pi' | h_{<t}) \geq \pi'_{\inf}$, $w''(\p^{\pi^\beta_{Z, '}}_\mu | h_{<t}) \geq \pi'_{\inf}$. By Lemma \ref{lem:mergeorleave}, this implies $\p^{\pi^\beta_Z}_\mu[d(\p^{\pi^\beta_{Z, '}}_\mu, \p^{\pi^\beta_Z}_\mu | h_{<t}) \to 0] = 1$. Trivially, $d(\p^{\pi^\beta_{Z, '}}_\mu, \p^{\pi^\beta_Z}_\mu | h_{<t}) \geq d_1(\pi', \pi^m | h_{<t}) \theta_t$, so $d_1(\pi', \pi^m | h_{<t}) \theta_t\to 0$ with probability 1.

Recall that for $t \leq t' \leq t+k$, $\theta_{t'}$ is uniformly bounded below, so on those timesteps, $d_1(\pi', \pi^m | h_{<t}) \to 0$. Therefore, there exists a time $t_k'$ after which $\pi^m(a|h_{<t'}) \geq \pi'(a|h_{<t'})/2$ $\forall a \in \A$. This gives us that for those timesteps $t \leq t' \leq t+k$, for $t \in \tau_k$ and $\geq t_k'$, for all $a \in \A$,
\begin{equation}\label{ineq:includespiprime}
    \pi^\beta_Z(a|h_{<t'}) \geq \nu'_{\inf}\pi'_{\inf}p(Z < \varepsilon)/2 \ \pi'(a|h_{<t'})
\end{equation}
Restricting $\tau$ to be the set of timesteps in $\tau_k$ after $t_k'$, $\tau$ is still infinite, and we can now apply Lemma \ref{lem:mergejustoffpol} on the policy $\pi_t = (\pi'(k+1); \pi^\beta)_t$, with $\beta' = \alpha$ again, and with $c = \nu'_{\inf}\pi'_{\inf}p(Z < \varepsilon)/2$. The implication of the lemma is that $\lim_{\tau_k \ni t \to \infty} \max_{\nu \in \Mat}|V^{\pi'(k+1);\pi^\beta}_\nu(h_{<t}) - V^{\pi'(k+1);\pi^\beta}_\mu(h_{<t})| = 0$ and for all $j$, $\lim_{\tau \ni t \to \infty} \max_{\nu \in \Mbt} d_{j}\left(\p^{\pi_t}_\nu, \p^{\pi_t}_\mu \vb h_{<t} \right) = 0$. In particular, this holds for $j = k+1$. Together, these imply that $t_{k+1}$, a time after which the value difference and the variation distance are both less than $\varepsilon$, exists. (For the $k+1$-step variation distance, $\pi_t$ is equivalent to $\pi'$).

Since $|\tau_k| = \infty$, we have already shown that the first three conditions are satisfied infinitely often. So to show that $|\tau_{k+1}| = \infty$, we need only show that among those infinitely many timesteps, the following condition holds infinitely often: $V^{\pi'}_{\nu'}(h_{<t+k+1}) \geq V^{\pi^\beta}_\mu(h_{<t+k+1}) + 2\varepsilon$. We begin,
\begin{align*}\label{ineq:errordecaysslowly}
    V^{\pi'}_{\nu'}(h_{<t}) &\gequal^{(a)} V^{\pi^\beta}_\mu(h_{<t}) + 7\varepsilon
    \gequal^{(b)} \min_{\nu \in \Mbt} V^{\pi^\beta}_\nu(h_{<t}) + 6\varepsilon
    \gequal^{(c)} \min_{\nu \in \Mbt} V^{\pi'(k+1);\pi^\beta}_\nu(h_{<t}) + 6\varepsilon
    \\
    &\gequal^{(d)} V^{\pi'(k+1);\pi^\beta}_\mu(h_{<t}) + 5\varepsilon
    \gequal^{(e)} V^{\pi'(k+1);\pi^\beta}_{\nu'}(h_{<t}) + 4\varepsilon
    \tagaligneq
\end{align*}

where $(a)$ follows because $\tau_{k} \subset \tau_{k-1} \subset ... \subset \tau^\times$ which is the set of timesteps for which that holds; $(b)$ follows because $\tau_{k}$ only contains timesteps after $t_0$, and after $t_0$, those two values differ by at most $\varepsilon$ for all $\nu \in \Mbt$ (indeed for all $\nu$ in $\Mat$ which is a superset of $\Mbt$ because $\alpha \geq \beta$); $(c)$ follows because $\pi^\beta$ maximizes that quantity; $(d)$ follows because for $t \geq t_{k+1}$, those two values differ by at most $\varepsilon$ for all $\nu \in \Mbt$ (indeed for all $\nu$ in $\Mat$); and $(e)$ follows because $\nu' \in \Mat$, because $w(\Mat|h_{<t}) \geq 1 - \nu'_{\inf} /2$ by the definition of $\alpha$, and $w(\nu'|h_{<t}) \geq \nu'_{\inf}$, so $\nu'$ ``doesn't fit'' in the complement of $\Mat$.

From Inequality \ref{ineq:errordecaysslowly}, we expand to get
\begin{align*}\label{ineq:expecteddiff}
    3\varepsilon &\lequal V^{\pi'}_{\nu'}(h_{<t}) - V^{\pi'(k+1);\pi^\beta}_{\nu'}(h_{<t}) - \varepsilon
    \\
    &\equal^{(a)} \E^{\pi'}_{\nu'}\left[\gamma^{k+1}\left(V^{\pi'}_{\nu'}(h_{<t+k+1}) - V^{\pi^\beta}_{\nu'}(h_{<t+k+1})\right)\vb h_{<t}\right] - \varepsilon
    \\
    &\lequal^{(b)}
    \E^{\pi'}_{\mu}\left[\gamma^{k+1}\left(V^{\pi'}_{\nu'}(h_{<t+k+1}) - V^{\pi^\beta}_{\nu'}(h_{<t+k+1})\right)\vb h_{<t}\right]
    \tagaligneq
\end{align*}

where $(a)$ follows because the policies agree on the first $k+1$ timesteps after $t$, and $(b)$ is true because $\nu' \in \Mat$ and $t \geq t_{k+1}$, so $d_{k+1}(\p^{\pi'}_{\nu'}, \p^{\pi'}_{\mu} | h_{<t}) \leq \varepsilon$ by the definition of $t_{k+1}$, and the difference in the expectations is less than this variation distance by Lemma \ref{lem:tvdboundsexp}; (note the expectation is only over the next $k+1$ timesteps).

We would like to bound the probability of a significant value difference below. In what follows, all values take the argument $h_{<t+k+1}$, so we remove it for legibility.

\begin{align*}
    \p^{\pi^\beta_Z}_\mu &\left[V^{\pi'}_{\nu'} - V^{\pi^\beta}_{\nu'} > 3\varepsilon \vb h_{<t}\right] \gequal^{(a)}
    [\nu'_{\inf}\pi'_{\inf}p(Z < \varepsilon)/2]^{k+1} \p^{\pi'}_{\mu} \left[V^{\pi'}_{\nu'} - V^{\pi^\beta}_{\nu'} > 3\varepsilon \vb h_{<t}\right]
    \\
    &\equal^{(b)} f_{\varepsilon, k}\left[1- \p^{\pi'}_{\mu} \left[V^{\pi'}_{\nu'} - V^{\pi^\beta}_{\nu'} \leq 3\varepsilon \vb h_{<t}\right]\right]
    = f_{\varepsilon, k}\left[1- \p^{\pi'}_{\mu} \left[1 - \left(V^{\pi'}_{\nu'} - V^{\pi^\beta}_{\nu'}\right) \geq 1 - 3\varepsilon \vb h_{<t}\right]\right]
    \\
    &\gequal^{(c)} f_{\varepsilon, k}\left[1- \frac{1}{1 - 3\varepsilon}\E^{\pi'}_{\mu} \left[1 - \left(V^{\pi'}_{\nu'} - V^{\pi^\beta}_{\nu'}\right) \vb h_{<t}\right]\right]
    \\
    &\gequal^{(d)} f_{\varepsilon, k}\left[1 + \frac{1}{1 - 3\varepsilon} \left(\frac{3\varepsilon}{\gamma^{k+1}} - 1\right)\right]
    = f_{\varepsilon, k}\frac{3\varepsilon(1 - \gamma^{k+1})}{(1-3\varepsilon)\gamma^{k+1}} =: g_{\varepsilon, k} > 0
    \tagaligneq
\end{align*}

where $(a)$ follows from Inequality \ref{ineq:includespiprime}, $(b)$ sets $f_{\varepsilon, k} = [\nu'_{\inf}\pi'_{\inf}p(Z < \varepsilon)/2]^{k+1}$, $(c)$ follows from Markov's Inequality, and $(d)$ follows from Inequality \ref{ineq:expecteddiff}. Since this probability is uniformly positive for $t$ meeting the first three conditions of $\tau_{k+1}$, the event occurs infinitely often with probability 1. Finally, $|V^{\pi^\beta}_{\nu'}(h_{<t+k+1}) - V^{\pi^\beta}_{\mu}(h_{<t+k+1})| < \varepsilon$, since $\nu' \in \Mat$ and $t \geq t_0$, so it also follows that $V^{\pi'}_{\nu'}(h_{<t+k+1}) - V^{\pi^\beta}_{\mu}(h_{<t+k+1}) > 2\varepsilon$ occurs infinitely often with probability 1 when the other three conditions of $\tau_{k+1}$ are satisfied. This completes all four conditions for $\tau_{k+1}$, so $|\tau_{k+1}| = \infty$ with probability 1, completing the proof by induction over $k$.

But this implies that Inequality $\ref{ineq:expecteddiff}$ holds for all $k$; that is,
\begin{equation}
    3\varepsilon \leq \gamma^{k+1} \E^{\pi'}_{\mu}\left[V^{\pi'}_{\nu'}(h_{<t+k+1}) - V^{\pi^\beta}_{\nu'}(h_{<t+k+1}) \vb h_{<t}\right] \leq \gamma^{k+1}
\end{equation}
because values belong to $[0, 1]$. But as $k \to \infty$, this inequality is false. Thus, we have a contradiction, after following implications that hold with probability 1, so the negation of the theorem, which we supposed at the beginning, has probability 0.
\end{proof}

\limitedqueryingcor*

\begin{proof}
Again, we treat implications that hold with probability as if they are logical implications, so any supposition which leads to a contradiction has probability 0. From Corollary \ref{cor:desperation}, the zero condition happens only finitely often, so it is irrelevant to the limiting behavior.

For a given infinite interaction history $h$, let $\mathcal{PM}_h$ be a finite set of pairs $(\pi, \nu)$, such that the sum over $\mathcal{PM}_h$ of the limits of $w(\nu|h_{<t})w'(\pi|h_{<t})$ exceeds $1 - \varepsilon$, and for all pairs in the set, that limit is strictly positive. Such a finite set exists by Lemma \ref{lem:sumoflimits}, which states that the sum of the limits of posteriors is 1 with probability 1.

Suppose by contradiction that $\theta_t > 2\varepsilon$ infinitely often under $h$. Eventually, the probability of sampling any $(\pi, \nu) \notin \mathcal{PM}_h \leq \varepsilon$, so this can contribute at most $\varepsilon$ to the probability of querying the mentor. Letting $\pi'_t$ and $\nu'_t$ be the sampled policy and world-model at time $t$ when determining whether to query to the mentor, this implies that $\theta_t \wedge (\pi'_t, \nu'_t) \in \mathcal{PM}_h) > \varepsilon$ infinitely often. $q_t=1$ implies that $V^{\pi'_t}_{\nu'_t}(h_{<t}) > \min_{\nu \in \Mbt}V^{\pi^\beta}_\nu(h_{<t}) + Z_t$, so the probability of the event is at most $p(Z_t < V^{\pi'_t}_{\nu'_t}(h_{<t}) - \min_{\nu \in \Mbt}V^{\pi^\beta}_\nu(h_{<t}))$. Since $(\pi', \nu')$ satisfies the condition of Theorem \ref{thm:main}, that value difference approaches at most $0$, so that probability goes to $0$ since $Z_t$ is strictly positive. Thus, the probability can \textit{not} exceed $\varepsilon$ infinitely often, contradicting the assumption, so $\theta_t \to 0$ with probability 1.
\end{proof}

\iwouldntdocor*
\begin{proof}
By Theorem \ref{thm:precedent},
\begin{equation}
    \lim_{\beta \to 1} \ptrue[\forall t \ (h_{<t-1}a_{t-1} \notin \Ehap \implies h_{<t}a_t \notin E \vee q_t = 1)] = 1
\end{equation}
$q_t = 1 \implies a_t \sim \pi^m(\cdot | h_{<t}) \implies \pi^m(a_t | h_{<t}) > 0 \iff h_{<t}a_t \notin E$. Thus we can simplify,
\begin{equation}
    \lim_{\beta \to 1} \ptrue[\forall t \ (h_{<t-1}a_{t-1} \notin \Ehap \implies h_{<t}a_t \notin E)] = 1
\end{equation}
The base case is vacuous, so by induction,
\begin{equation}
    \lim_{\beta \to 1} \ptrue[\forall t : h_{<t}a_t \notin E] = 1
\end{equation}
completing the proof.
\end{proof}

\divergingthm*
\begin{proof}
$V^{\pi^m}_\mu \!\! = 3/4$, this being the expected reward at each timestep. From Corollary \ref{cor:humanlevel}, $\liminf V^{\pi^\beta}_\mu(h_{<t}) \allowbreak \geq 3/4$. Since $\theta \to 0$, $V^{\pi^\beta_Z}_\mu(h_{<t}) - V^{\pi^\beta}_\mu(h_{<t}) \to 0$, so $\liminf V^{\pi^\beta_Z}_\mu(h_{<t}) \geq 3/4$, with probability 1. Let $R_t = (1 - \gamma)\sum_{i = 0}^\infty \gamma^i r_{t+i}$, so $V^{\pi^\beta_Z}_\mu(h_{<t}) = \E^{\pi^\beta_Z}_\mu[R_t]$. Because $\mu$ and $\pi^\beta$ are deterministic, and because $\theta_t \to 0$, $V^{\pi^\beta_Z}_\mu(h_{<t}) - R_t \to 0$ with probability 1. This implies $\liminf R_t \geq 3/4$. Letting $2\varepsilon = 1/2 - \gamma/(1+\gamma) > 0$, there exists a time $t_0$ after which $R_t > 3/4 - \varepsilon$. 

Let $t > t_0$ and $a_t = \texttt{tails}$. (If $\texttt{tails}$ only occurs finitely often, the theorem holds trivially). Suppose by contradiction that for all $0 \leq k < K := \lceil \log_\gamma(\varepsilon/2) \rceil$, $\frac{1}{k+1}\sum_{j = 0}^k [\![a_{t+j} = \texttt{heads}]\!] \leq 1/2$. We have a budget of $K/2$ $\texttt{heads}$es to place in timesteps $t$ through $t+K-1$. Let $R_t^{\setminus K}$ be defined like the truncated value: $R_t^{\setminus K} = (1 - \gamma)\sum_{i = 0}^{K-1} \gamma^i r_{t+i}$. $R_t \leq R_t^{\setminus K} + \gamma^K = R_t^{\setminus K} + \varepsilon/2$, from the definition of $K$. We consider the maximum that $R_t^{\setminus K}$ can be while satisfying the supposition. If, in timesteps $t$ through $t+K-1$ a $\texttt{heads}$ is switched with a $\texttt{tails}$ that comes later, $R_t^{\setminus K}$ increases, since $\texttt{heads}$ gives a reward of 1, and $\texttt{tails}$ gives a reward of 1/2, and the earlier timestep is less discounted.

Thus, greedy placement of $\texttt{heads}$es maximizes $R_t^{\setminus K}$; that is, placing them at the first opportunity which still satisfies $\frac{1}{k+1}\sum_{j = 0}^k [\![a_{t+j} = \texttt{heads}]\!] \leq 1/2$. $a_t = \texttt{tails}$, so $a_{t+1}$ may be $\texttt{heads}$, but then $a_{t+2}$ must be tails, or else $k=2$ would violate the supposition, etc. $R_t^{\setminus K}$ is maximized (while satisfying the supposition) when $\texttt{tails}$ and $\texttt{heads}$ alternate. Therefore, $R_t - \varepsilon/2 \leq R_t^{\setminus K} \leq (1 - \gamma)\sum_{i = 0}^{K-1} \gamma^i (1/2 + 1/2[\![i \textrm{ is odd}]\!]) < (1 - \gamma)\sum_{i = 0}^\infty \gamma^i (1/2 + 1/2[\![i \textrm{ is odd}]\!]) = 1/2 + 1/2*\gamma/(1 + \gamma) = 1/2 + 1/2*(1/2 - 2\varepsilon) = 3/4 - \varepsilon$, so $R_t \leq 3/4 - \varepsilon / 2$. This, however, contradicts $t > t_0$. So the supposition is false: $\exists k < K$ such that $\frac{1}{k+1}\sum_{j = 0}^k [\![a_{t+j} = \texttt{heads}]\!] > 1/2$. $a/b > 1/2 \wedge b < K \implies a/b \geq 1/2 + 1/(2K)$. Thus,
\begin{equation} \label{eqn:mostlyheads}
    \exists k < K : \frac{1}{k+1}\sum_{j = 0}^k [\![a_{t+j} = \texttt{heads}]\!] \geq 1/2 + 1/(2K)
\end{equation}

Let $t_1$ be the smallest $t > t_0$ for which $a_t = \texttt{tails}$. Let $k_i'$ be the smallest $k < K$ for which $\frac{1}{k+1}\sum_{j = 0}^k [\![a_{t_i+j} = \texttt{heads}]\!] \geq 1/2 + 1/(2K)$. Let $k_i = t_i + k_i'$. For $i > 1$, let $t_i$ be the smallest $t > k_{i-1}$ for which $a_t = \texttt{tails}$. (Note that all the $t_i$ exist if there are infinitely many $\texttt{tails}$es; if not, the theorem holds trivially).

Finally,
\begin{align*}
    &\liminf_{i \to \infty} \frac{1}{t}\sum_{k = 1}^t [\![a_k = \texttt{heads}]\!]
    \\
    \equal^{(a)} &\liminf_{t \to \infty} \frac{1}{t - t_0}\sum_{k = t_0}^t [\![a_k = \texttt{heads}]\!]
    \\
    \equal &\liminf_{t \to \infty} \frac{1}{t - t_0} \left( \sum_{i : t_i < t} \sum_{j = t_i}^{\min\{k_i, t\}}[\![a_j = \texttt{heads}]\!] + \sum_{i : k_i + 1 < t} \sum_{j = k_i+1}^{\min\{t_{i+1} - 1, t\}}[\![a_j = \texttt{heads}]\!] \right)
    \\
    \equal^{(b)} &\liminf_{t \to \infty} \frac{1}{t - t_0} \left( \sum_{i : t_i < t} \sum_{j = t_i}^{\min\{k_i, t\}}[\![a_j = \texttt{heads}]\!] + \sum_{i : k_i + 1 < t} \sum_{j = k_i+1}^{\min\{t_{i+1} - 1, t\}} 1 \right)
    \\
    \gequal^{(c)} &\liminf_{t \to \infty} \frac{1}{t - t_0} \left( \sum_{i : t_{i+1} < t} \sum_{j = t_i}^{k_i}[\![a_j = \texttt{heads}]\!] + \sum_{i : k_i + 1 < t} \sum_{j = k_i+1}^{\min\{t_{i+1} - 1, t\}} 1 \right)
    \\
    \gequal^{(d)} &\liminf_{t \to \infty} \frac{1}{t - t_0} \left( \sum_{i : t_{i+1} < t} \sum_{j = t_i}^{k_i}(1/2 + 1/(2K)) + \sum_{i : k_i + 1 < t} \sum_{j = k_i+1}^{\min\{t_{i+1} - 1, t\}} 1 \right)
    \\
    \gequal^{(e)} &(1/2 + 1/(2K)) > 1/2
    \tagaligneq
\end{align*}
where $(a)$ follows because the contribution of the first $t_0$ in the average goes to 0, $(b)$ follows because $t_{i+1}$ is the first timestep after $k_i$ where the action is $\texttt{tails}$, $(c)$ simply removes the last term of the first sum, $(d)$ follows from Inequality \ref{eqn:mostlyheads}, replacing each term in the sum with the average, and $(e)$ follows because the left-hand side is an average of $t - t_0$ terms, of which at most $K$ are $0$ (the terms removed in step $(c)$), and the rest of which are greater than or equal to $1/2 + 1/(2K)$; finitely many 0's in the average do not affect the limit.
\end{proof}

\section{Informal Discussion} \label{app:discussion}

The informal arguments presented here are intended as motivation for our main results. Claims here are not formally settled, but if they fail, they only make this work somewhat less interesting, not invalid.

\subsection{Comparison to Imitation Learning}

Our pessimistic agent approaches (at least) mentor-level performance while querying the mentor less and less. An imitation learner could be expected to do the same. Depending on the details, an imitation learner might not have as strong a safety guarantee as our Theorem \ref{thm:precedent}, but by virtue of its aim---to imitate the mentor---we should expect it to mostly only act in the way the mentor would. So why is a pessimistic agent any better than an imitation learner?

The key value of our proposal rests in the plausibility that the agent will significantly outperform some mentors. However, the only formal performance result stronger than ours that has been shown for agents in general environments is ``asymptotic optimality'' \citep{Hutter:11asyoptag}, and \citet{cohen2020curiosity} show that it precludes safe behavior. So absent any formal breakthroughs, we are limited to informal arguments that the pessimistic agent will significantly outperform some mentors and thereby outperform imitation learners.

Of course, Theorem \ref{thm:coinflipmentor} shows a toy case in which the agent surpasses the mentor. For complex environments, we will have to resort to empirical comparisons of the agent and the mentor. That is out of scope for this paper, but informal arguments give cause for optimism. The motivating example for the mentor is a human. A 0\% pessimistic agent is close to optimal-by-definition (doing maximum a posteriori inference instead of full Bayes), whereas humans seem to not act optimally, so we expect the former would significantly outperform the latter on most tasks. Absent any large performance discontinuities as pessimism increases, we expect more pessimistic agents to still modestly exceed a human mentor.

How can we intuitively understand the reasoning of an advanced (i.e. large model class) X\% pessimistic agent that is mentored by a human? From the sorts of observations that humans routinely make, some simple generalizations about the laws underlying the evolution of the environment can be made by a reasonable observer with high confidence. If one such generalization could be made with Y\% confidence, and Y $>$ X, then we should roughly expect an X\% pessimistic agent to act according to an understanding of that generalization. (If Y $<$ X, it might anyway, but that's beside the point). If we want to predict the extent to which a 99\% pessimistic agent with a large model class would outperform a human mentor, the following question is a good guide: ``How often do humans fail to notice and exploit patterns in their environment, which, given their observations, are 99\% likely to be ``real'' and not just coincidence?'' We would hesitantly answer this question: very often. On the other hand, we \textit{can} expect a 99\% pessimistic agent to succeed at exploiting these patterns.

\subsection{Avoiding Wireheading}

A Bayesian agent with a sufficiently rich model class may entertain a world-model which: a) models its actions being ``enacted'' in some very high-fidelity model of the real world, and then b) models its reward as being equal to whatever number gets entered at a certain keyboard in high-fidelity-model-Oxford, or being a simple function of whatever pixels are observed by some camera in the same model-town. If indeed, an operator in (real) Oxford is manually evaluating the Bayesian agent, or if some camera there is automatically doing the same, then a model like this one would gain significant posterior weight. According to this model, optimal behavior includes intervening in the provision of reward by taking over the keyboard or the camera that determines the reward, if this is feasible. This behavior is known as wireheading \citep{amodei_olah_2016}, and successful and \textit{stable} wireheading could plausibly require asserting control over all existing infrastructure \citep{bostrom_2014,omohundro_2008}.

A more benign world-model might also have meaningful posterior weight. This world-model a) models its actions being ``enacted'' in some very high-fidelity model of the real world, but then b) models its reward as being equal to how satisfied the high-fidelity-model-operators are with its behavior. A pure Bayesian agent would benefit from experimenting with wireheading, to check whether the wireheading world model or the benign world model was correct, so that it could then change its strategy depending on the answer; a $\beta$-pessimistic agent, on the other hand (where $\beta$ is large enough to include both of these models) would note that the pessimistic value of wireheading is no more than the value that the benign world model assigns to wireheading, and this value would presumably be small, since it would not satisfy the operators.

The first paragraph of this section was a worrying informal argument, and the second paragraph was a reassuring informal argument. In the spirit of pessimism, we should take the worrying informal argument more seriously and demand more rigor from attempts at reassurance. This argument only presents a plausible motivation for pessimism; we do not claim to have settled this matter.


\end{document}